\documentclass{article}

\usepackage[main, final]{neurips_2026}

\usepackage[utf8]{inputenc} 
\usepackage[T1]{fontenc}    
\usepackage{hyperref}       
\usepackage{url}            
\usepackage{booktabs}       
\usepackage{amsfonts}       
\usepackage{nicefrac}       
\usepackage{microtype}      
\usepackage{xcolor}         

\usepackage{multirow}       
\usepackage{makecell}       
\usepackage{amsmath}        
\usepackage{algorithm}      
\usepackage{algorithmic}    
\usepackage{enumitem}       

\usepackage{graphicx}
\usepackage{subcaption}
\usepackage{wrapfig}
\usepackage{caption}
\newtheorem{theorem}{Theorem}[section]

\newenvironment{proof}{{\noindent\bf Proof.}}{\hfill $\square$\par}
\newenvironment{remark}{{\noindent\bf Remark.}}{}

\title{Towards Scalable Data Diversification for Language Model Pretraining via Leverage Score Sampling}

\author{%
  Zailin Ma$^{1,2}$\thanks{Email: mazailin@stu.pku.edu.cn. Work done during an internship at Wizard Intelligence Learning Lab.} \quad
  Quzhe Huang$^{2}$ \quad
  Yujun Li$^{2}$ \quad
  Congyuan Rao$^{3}$ \quad
  Yaodong Yang$^{1}$\thanks{Corresponding author: yaodong.yang@pku.edu.cn} \\[1em]
  $^{1}$Peking University \quad
  $^{2}$Wizard Intelligence Learning Lab \quad
  $^{3}$Tsinghua University
}

\begin{document}

\maketitle

\begin{abstract}
  Data selection for language model pretraining faces a fundamental tension between quality and diversity. While quality filtering is empirically effective, it often induces diversity collapse: by favoring texts similar to high-quality reference corpora (e.g., educational or QA-style data), it systematically excludes valuable data from underrepresented domains. In contrast, diversified selection preserves domain balance and encourages robust downstream performance, yet existing methods either focus on coverage-oriented objectives that indirectly enhance diversity, or directly optimize for diversity via costly covariance matrix recomputation that limits scalability. To address these issues, we introduce \textbf{Leverage Score Sampling (Lev)}, which iteratively selects samples that maximally expand the determinantal volume of the embedded data via leverage scores, a computationally efficient criterion that eliminates matrix recomputation and enables scalable selection. Empirically, Lev delivers up to $72\times$ speedup and improves dataset diversity, measured by the Vendi score, by $9.2\%$ over the strong diversification baseline \textbf{DiSF}. On CommonCrawl (CC) web data selection, Lev improves accuracy across seven downstream tasks by up to $1.31\%$ over existing baselines. For domains where robust quality criteria are inherently difficult to define (e.g., code), Lev serves as an effective unsupervised curation alternative: on StarCoderData, the selected subset reduces bits-per-byte by $3.08\%$ over DiSF. Notably, we uncover a cross-domain collapse of quality filtering: CC data filtered by DCLM-fastText fail to retain sufficient code-related content, yielding inferior code performance relative to Lev-selected data. These findings advocate for integrating diversity-aware practices into quality filtering for more effective data curation in language model pretraining.
\end{abstract}

\begin{figure}[h]
    \centering
    \vspace{-10pt}
    \includegraphics[width=0.93\linewidth]{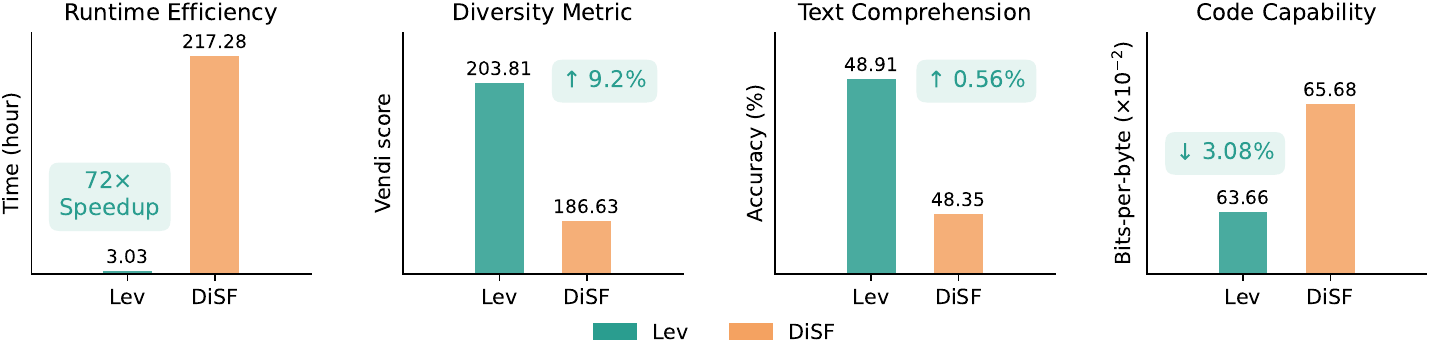}
    \vspace{-5pt}
    \caption{Lev outperforms existing diversification baseline DiSF across all fronts.}
    \vspace{-15pt}
\end{figure}

\section{Introduction}
\label{sec:introduction}

Large language models (LLMs) have achieved unprecedented performance in general text comprehension and problem-solving tasks~\citep{deepseekai2026deepseekv4,kimiteam2026kimik25visualagentic}, with their capabilities critically hinging on the composition of pretraining data. Extensive research has sought to curate higher-quality corpora from raw web sources such as CommonCrawl (CC)~\citep{soldaini2024dolma,penedo2023refinedweb}. Among these efforts, model-based quality filtering~\citep{penedo2024fineweb,wettig2024qurating} has emerged as the dominant mechanism owing to its superior empirical performance. Practitioners train lightweight classifiers such as fastText~\citep{joulin2016bagtricksefficienttext} using curated question-answering or LLM-labeled high-quality texts as positive samples~\citep{li2024datacomp,wang2025ultrafinewebefficientdatafiltering}. However, this approach induces \textbf{dimensional collapse}~\citep{fan2025combatting}: the embedding vectors of the filtered data span a low-dimensional subspace, indicating severe diversity degradation from a geometric perspective, as the classifier favors data similar to the narrow references but neglects out-of-domain instances. Consequently, the resulting models exhibit performance degradation on underrepresented domains of the curated corpora.

To mitigate this degradation, diversified data selection has emerged as a natural solution. Early efforts address this through domain-level organization~\citep{wettig2025organize} or attribute-level balancing~\citep{zhang2025harnessing,liu2025quadmix}, which only indirectly alleviate dimensional collapse by operating at a coarse granularity. Finer-grained approaches, such as coverage-based sampling~\citep{shao2024balanceddatasamplinglanguage,sachdeva2026askllm}, pursue uniformity over the data distribution; yet uniform density does not guarantee that selected samples span the full representation space needed for geometric diversity. To directly optimize set-level geometric diversity, DiSF~\citep{fan2025combatting} greedily selects samples that minimize the Frobenius norm of the covariance matrix, achieving strong downstream performance. However, this procedure requires per-candidate covariance recomputation at every iteration, incurring prohibitive computational overhead that severely limits scalability.

Motivated by these limitations, we propose \textbf{Leverage Score Sampling (Lev)}, a scalable scheme that employs leverage scores~\citep{7439920,10.5555/3327144.3327172,9747039} as a principled metric for diversity improvement. Grounded in the geometric insight that dataset diversity corresponds to the determinantal volume of the representation subspace spanned by embedded samples~\citep{wang2024diversity}, we prove that the leverage score of a candidate exactly quantifies its contribution to expanding this volume upon inclusion in the selected subset (Theorem~\ref{geothm}). Building on this result, Lev iteratively selects samples with the highest leverage scores to directly maximize geometric diversity (Figure~\ref{geointuition}). Unlike DiSF, which requires costly covariance matrix recomputation for each candidate to evaluate its marginal diversity contribution, Lev evaluates all candidates via quadratic expressions with a single shared Gram matrix, achieving superior computational efficiency and enabling scalable applications.

\begin{figure}[t]
    \centering
    \includegraphics[width=\linewidth]{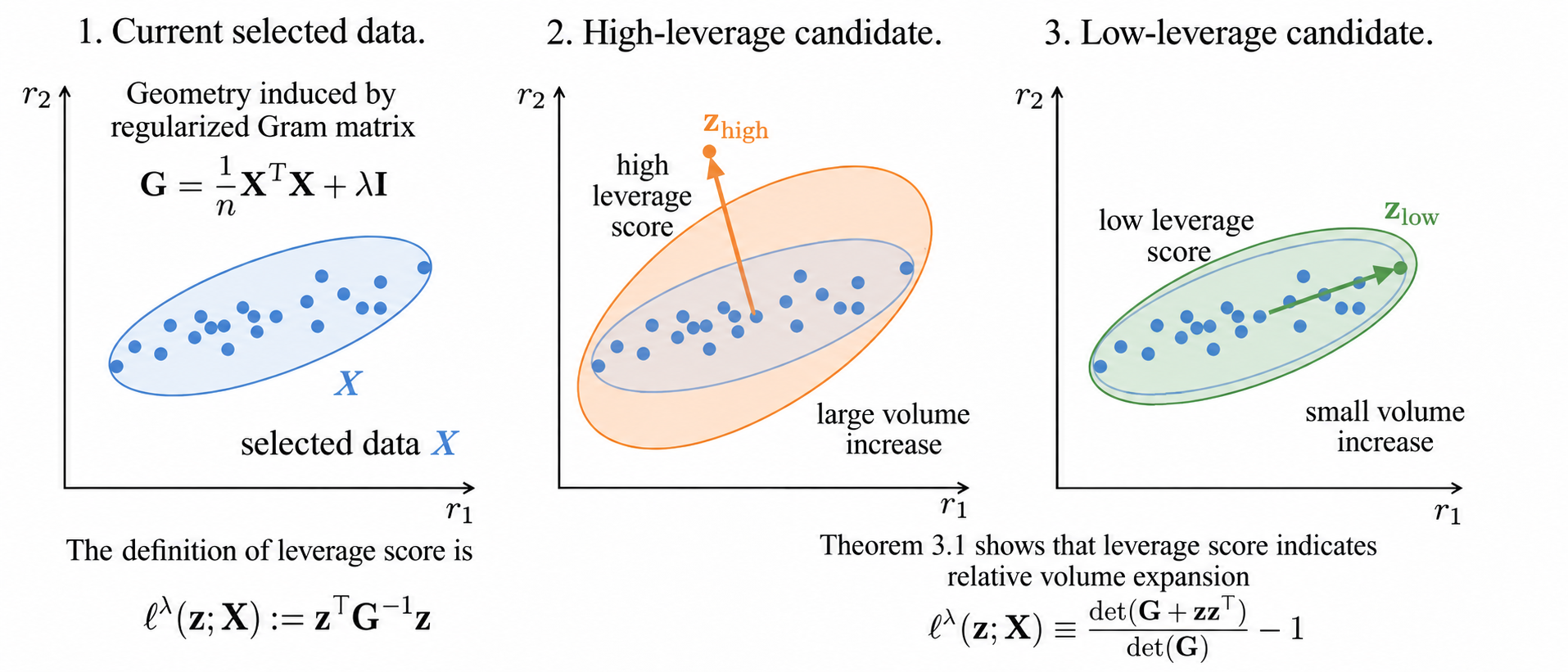}
    \caption{Leverage score is efficiently computed via a quadratic form (left). It quantifies the relative volume expansion of the candidate sample once selected (right), enabling effective diversification.}
    \label{geointuition}
\end{figure}

We validate Lev through extensive experiments. Compared to the strong diversification baseline DiSF, it improves dataset diversity by 9.2\% while achieving a $72\times$ runtime speedup. On CommonCrawl (CC) web data subset selection, it improves average accuracy across seven downstream tasks by up to $1.31\%$ over existing baselines. Crucially, Lev requires no curated quality references, making it uniquely effective for domains where defining quality is inherently ambiguous (e.g., code): on StarCoderData, it reduces bits-per-byte by $3.08\%$ over DiSF. 
Beyond these gains, our analysis uncovers that the prevailing quality filter DCLM-fastText~\citep{li2024datacomp} induces a \textbf{code-specific dimensional collapse}: it insufficiently retains code-related content from CC compared to Lev, causing persistent code inferiority relative to diversity-based selection. 
Our contributions are summarized as follows.
\begin{itemize}[leftmargin=1.4em, itemsep=0pt, topsep=0pt]
    \item We establish leverage scores as a principled metric for diversity enhancement and propose \textbf{Leverage Score Sampling (Lev)}, a scalable diversification algorithm. We show that Lev is provably more efficient than DiSF and demonstrate up to $72\times$ speedup in practice.
    \item We demonstrate that Lev achieves superior downstream performance over existing baselines on CommonCrawl web data selection. For domains lacking robust quality criteria (e.g., code), we show that Lev serves as an effective \textbf{unsupervised} curation strategy: on StarCoderData, Lev selects subset that reduces bits-per-byte by $3.08\%$ over DiSF.
    \item We uncover that quality filtering induces cross-domain dimensional collapse, disproportionately excluding code-related content that diversity-based methods effectively preserve. This finding underscores the necessity of integrating diversity-aware selection into quality filtering for more robust data curation.
\end{itemize}

\section{Preliminaries}
\subsection{Problem statement}
\label{subsec:problem_statement}

We study the problem of selecting a diverse subset of pretraining data for language models subject to a fixed selection ratio. Let $\mathcal{D} = \{c_1, \dots, c_D\}$ denote a corpus where each $c_i$ is a document. Examples include heuristically filtered and deduplicated web text (e.g., CommonCrawl) or source code (e.g., GitHub). Each document $c_i$ is mapped to a $d$-dimensional \textit{unit vector} via an embedding model $\phi$, i.e., $x_i = \phi(c_i) \in \mathbb{R}^{d}$ and $\|x_i\|_2=1$. Given a target selection ratio $\alpha$, we formulate the data selection problem as:
\begin{equation} 
\label{eq:problem} 
\mathcal{S}^* = \mathop{\mathrm{argmax}}_{\substack{\mathcal{S} \subset \mathcal{D} \\ |\mathcal{S}| = \alpha D}} \mathrm{Diversity}(\mathcal{S}), 
\end{equation} 
where $\mathrm{Diversity}(\mathcal{S})$ denotes a set function that quantifies the diversity of the subset $\mathcal{S}$. 

\subsection{Recent progress}
\label{recentprogress}
Existing approaches to diversity-aware data curation can be broadly categorized into heuristic-guided diversity methods and methods that directly optimize set-level diversity. A large body of prior work falls into the former category, encouraging diversity through indirect signals such as domain-level organization~\citep{wettig2025organize}, attribute-level balancing based on aggregated quality signals~\citep{zhang2025harnessing,liu2025quadmix,fan2025joint}, or coverage-based sampling that promotes uniform coverage via clustering or local density estimates~\citep{shao2024balanceddatasamplinglanguage,sachdeva2026askllm}. Among these approaches, coverage-based methods come closest to capturing sample-level diversity. However, density-based uniformity does not necessarily imply geometric diversity: in high-dimensional spaces, uniformly distributed samples may still lie in a low-dimensional subspace, thereby failing to expand the representational subspace. In contrast, geometry-based methods explicitly optimize $\mathrm{Diversity}(\mathcal{S})$ in Eq.~\eqref{eq:problem}. For instance, DiSF~\citep{fan2025combatting} greedily selects $x_i$ by maximizing $f(x_i) = e^{-\|C(\mathcal{S} \cup \{x_i\})\|_F}$, where $C(\mathcal{S})$ denotes the normalized feature covariance matrix, encouraging a more uniform eigenspectrum. However, evaluating this score requires expensive iterative covariance updates, which severely limits scalability. This computational bottleneck motivates us to seek an alternative geometric objective that directly spans the representation subspace without repeated covariance tracking. A broader review of related work is provided in Appendix~\ref{subsec:recent_progress}.

\section{Method}
\label{sec:method}

\subsection{Leverage score as diversity signal}
\label{subsec:leverage_scores}

To address this computational bottleneck, we introduce leverage scores as an effective and scalable diversity signal. Leverage scores were originally used for scalable kernel approximation in ridge regression~\citep{7439920,10.5555/3327144.3327172,9747039,10.5555/3294996.3295140,10.5555/2969239.2969326}, where they identify samples that most effectively span the representational space for efficient Gram matrix approximation. Grounded in this geometric insight, we use leverage scores as a principled measurement for dataset diversity spanning. Specifically, let $X = [x_1,\ldots,x_n]^\top \in \mathbb{R}^{n\times d}$ denote the embedding vectors of the selected data pool. For a candicate sample with embedding $z \in \mathbb{R}^d$, the leverage score of $z$ with respect to $X$ is defined as
\begin{equation}
\label{eq:leverage_score}
  \ell^{\lambda}(z;X) := z^\top \left(\frac{1}{n}X^\top X + \lambda I_d\right)^{-1} z,
\end{equation}
where $\lambda > 0$ is a regularization constant that ensures numerical stability (e.g., $\lambda=10^{-5}$). Let $G=\frac{1}{n}X^\top X + \lambda I_d$ be the corresponding regularized Gram matrix. We provide Theorem~\ref{geothm} that rigorously describes the geometric insight of leverage scores and Figure~\ref{geointuition} for a concrete illustration.
\begin{theorem}
    \label{geothm}
    For a candidate sample $z\in \mathbb{R}^d$, the following relationship holds:
    \begin{equation}
    \ell^{\lambda}(z;X) = \frac{\det(G+z z^\top)}{\det(G)}-1.
    \end{equation}
\end{theorem}
The determinant of the regularized Gram matrix $G$ can be interpreted as the squared volume of the polytope spanned by $X$ in the representation space~\citep{doi:10.1137/1.9781421407944}. The theorem reveals that the leverage score directly quantifies the relative volume expansion induced by adding the candidate sample $z$ to the current data pool. Samples with higher leverage scores yield larger volume expansion, and therefore contribute more substantially to the geometric diversity of the selected set.

\subsection{Algorithm}
\label{subsec:algorithm}

\begin{algorithm}[tbp]
\caption{Leverage Score Sampling for Diversified Data Selection}
\label{alg:leverage_sampling}
\begin{algorithmic}[1]
\STATE \textbf{Input:} Full corpus embeddings $\mathcal{Z} = \{z_1,\ldots,z_D\} \subset \mathbb{R}^d$, selection ratio $\alpha \in (0,1)$, batch size $B$, step size $b$, regularization $\lambda$
\STATE \textbf{Output:} Selected indices $\mathcal{S}$
\STATE $M \leftarrow \lceil D / B \rceil$ \COMMENT{batch number}
\FOR{each batch $\mathcal{B} \subset \{\mathcal{B}_1,\ldots,\mathcal{B}_M\}$}
    \STATE Randomly select $b$ indices from $\mathcal{B}$, add to $\mathcal{S}$
    \STATE $K \leftarrow \sum_{i \in \mathcal{S}} z_i z_i^\top$ \COMMENT{initialize Gram matrix}
    \FOR{$t = 1$ to $\lceil \alpha B / b \rceil-1$}
        \STATE $G \leftarrow \frac{1}{|\mathcal{S}|}K + \lambda I_d$ \COMMENT{regularized Gram matrix}
        \STATE $L \leftarrow \text{Cholesky}(G)$ \COMMENT{Cholesky decomposition, $G = LL^\top$}
        \FOR{each candidate $i \in \mathcal{B}$}
            \STATE $w \leftarrow L^{-1} z_i$ \COMMENT{solve triangular system}
            \STATE $\ell_i \leftarrow \|w\|_2^2$ \COMMENT{leverage score}
        \ENDFOR
        \STATE Sort candidates in $\mathcal{B}$ by leverage scores in descending order
        \STATE Select top-$b$ indices not in $\mathcal{S}$, denote as $\mathcal{I}_{\text{new}}$
        \STATE $\mathcal{S} \leftarrow \mathcal{S} \cup \mathcal{I}_{\text{new}}$
        \STATE $K \leftarrow K + \sum_{i \in \mathcal{I}_{\text{new}}} z_i z_i^\top$ \COMMENT{update Gram matrix}
    \ENDFOR
\ENDFOR
\STATE \textbf{return} $\mathcal{S}$
\end{algorithmic}
\end{algorithm}

We present our diversified data selection method \textit{Leverage Score Sampling (Lev)} in Algorithm~\ref{alg:leverage_sampling}. Let $D$ denote the corpus size, $B$ the batch size, $b$ the step size, $d$ the embedding dimension and $\alpha$ the selection ratio. In practice, $B$ and $d$ are typically of the same order of magnitude. The algorithm processes the full corpus in batches of size $B$, requiring approximately $D/B$ batch iterations. Within each batch, we perform $m = \alpha B / b $ sub-iterations to select $b$ samples at each step, resulting in $\alpha B$ samples in total. At each iteration, the leverage scores of candidate samples are computed relative to the selected data pool $X$, and samples with the highest score are added to $X$. Then we update the Gram matrix with the newly selected data for the next round of selection. The algorithm has two key computational advantages. First, we employ Cholesky decomposition~\citep{doi:10.1137/1.9781421407944} on the Gram matrix to compute leverage scores. Instead of directly inverting $\Sigma$, we solve $Lw = z_i$ and compute $\ell_i = \|w\|_2^2$, which is algebraically equivalent to computing $\ell_i = z_i^\top \Sigma^{-1} z_i$, but is numerically more stable and efficient. Second, the computation of leverage scores for all candidate samples rely on a single shared Gram matrix $G$, which can be efficiently maintained by incrementally updating the feature covariance matrix $K$ through the selection process. This obviates the need to recompute feature covariance matrix for every candidate sample, thus achieving significantly better computational efficiency than DiSF~\cite{fan2025combatting}. Codes are available at: \url{https://github.com/mazailin/LevDataDiversity}.

\subsection{How do leverage scores identify underrepresented semantics?}
\label{subsec:mathematical_interpretation}

To further dive into the mechanism of how exactly leverage score works, we consider the structure of the Gram matrix. Let $X/\sqrt{n} = U\Sigma V^\top$ be the singular value decomposition, where $U$ and $V$ are orthonormal matrices and $\Sigma=\text{diag}\{\sigma_1,\ldots,\sigma_d\}$. Then $\sigma_i^2$ is the $i$-th eigenvalue of $X^\top X/n$. We define the columns $v_i$ of $V$ as the \textit{principal semantic vectors} of the current data pool $X$. Because $\sum_{i=1}^{d}\sigma_i^2 = 1$, the eigenvalue $\sigma_i^2$ reflects the \textit{occurrence frequencies} of $v_i$ (see Appendix \ref{intfre} for details). More frequently occured semantics share larger values of $\sigma_i^2$. For a new candidate $z$, let $\tilde{z}_i = v_i^\top z$ denote its projection onto the $i$-th principal semantic vector. We have the following result.

\begin{theorem}
    \label{reprethm}
    The leverage score function has a decomposed expression as follows.
    \begin{equation}
    \ell^{\lambda}(z;X) = \sum_{i=1}^{d} \frac{\tilde{z}_i^2}{\sigma_i^2+\lambda}.
    \end{equation}
\end{theorem}

\begin{figure}[tbp]
    \centering
    \includegraphics[width=\linewidth]{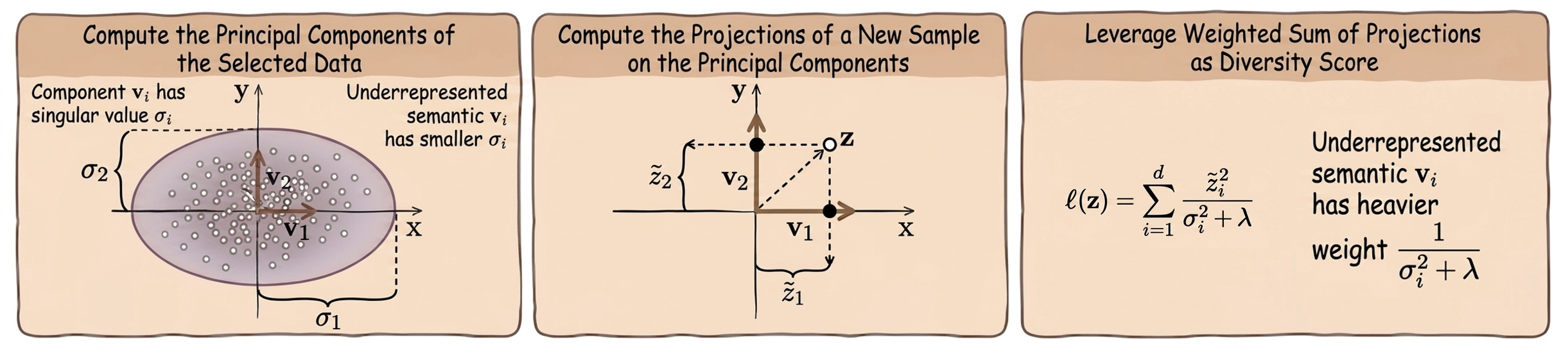}
    \caption{The statistical mechanism of leverage score functions: underrepresented semantic components are assigned higher weights, highlighting samples that stretch along these components.}
\end{figure}

This theorem reveals that $\ell^{\lambda}(z;X)$ aggregates the contributions of $z$ across all principal semantic vectors, weighted by the \textit{inverse} of their occurrence frequencies and the regularization parameter \(\lambda \ll 1\). If $v_i$ has almost never occurred in $X$, then $\sigma_i^2 \approx 0$. Thus the magnitude of $1/(\sigma_i^2+\lambda)\approx 1/\lambda$ becomes significant such that it effectively highlights samples that cover more semantics along $v_i$. If $\sigma_i^2 > 0$, then $1/(\sigma_i^2+\lambda)\approx 1/\sigma_i^2$ is dominated by the occurrence frequency, where semantics with lower frequencies share heavier weights. Grounded in this statistical mechanism, the score is named after \textit{leverage}. Empirically, setting \(\lambda = (100\times d)^{-1}\) is sufficient to highlight the uncovered semantic signals. We provide the proofs of both theorems and a detailed discussion on \(\lambda\) in Appendex~\ref{math}.

\section{Experiments}
\label{sec:experiments}

\vspace{-3pt}
\subsection{Setup}
\label{subsec:setup}
\vspace{-2pt}

\textbf{Datasets and models.} We consider two domain data, web and code. For web data, we process a \textbf{CommonCrawl (CC)} dump with \texttt{datatrove}~\cite{penedo2024fineweb} following similar pipeline steps in FineWeb~\cite{penedo2024fineweb}, including heuristic filtering and deduplication and excluding quality filtering to preserve the original degree of data diversity. For code data, we employ \textbf{StarCoderData}~\cite{li2023starcoder} and restrict our experiments to six programming languages including Python, Java, JavaScript, Go, PHP, and Ruby. We adopt the model architecture of Llama 3.1~\citep{grattafiori2024llama3herdmodels} with an adapted configuration of 1.1B parameters detailed in Appendix~\ref{model_details}, and consider a training budget of 60B token data sampled by various strategies. With a batch size of 1024 and a context length of 2048, this configuration results in 28610 training iterations. The training configuration is applied for both the main experiments and the ablation study.

\textbf{Baselines.} For diversity-based methods, we include \textbf{DiSF}~\citep{fan2025combatting} that maximizes the uniformity of the eigenvalues of the data covariance matrix, and \textbf{Density} sampling~\citep{sachdeva2026askllm} that uses density estimation for uniform sampling over the original corpus. \textbf{DSIR}~\citep{xie2023data} matches samples with targeted domains, \textit{Wikipedia} and \textit{Book}, using n-gram features. \textbf{Random} sampling provides a non-selective baseline. 
Additionally, we apply \textbf{DCLM-fastText}~\cite{li2024datacomp} to obtain high-quality subsets from CC for comparative experiments. DCLM-fastText is a strong quality scorer widely used for CC quality filtering. It is trained on positive samples of instruction-formatted and question-answer style text spanning a broad range of topics, along with negative samples from the original RefinedWeb~\citep{penedo2023refinedweb}.
Our method, \textbf{Lev}, introduces the leverage score function as a principled diversity measurement. For each baseline, we consider two selection ratios, 40\% and 50\%, for empirical validations.

\textbf{Evaluations.} We evaluate all models using the OLMES framework \citep{gu2025olmes}. For natural language understanding, we assess performance on seven established benchmarks: \textbf{ARC-easy}, \textbf{ARC-challenge} \citep{clark2018think}, \textbf{WinoGrande} \citep{sakaguchi2020winogrande}, \textbf{HellaSwag} \citep{zellers2019hellaswag}, \textbf{PIQA} \citep{bisk2020piqa}, \textbf{Social IQa} \citep{sap2019socialiqa} and \textbf{CommonSenseQA} \citep{talmor2019commonsenseqa}. The choice with the lowest perplexity is considered the response of the model and the overall accuracy is the primary metric. We define the mean score of the seven tasks as \textbf{Base Mean}, describing the base language understanding ability of a model. For code evaluation, we adopt tasks of six languages in \textbf{Multilingual MBPP}~\citep{cassano2022multiplescalableextensibleapproach} that correspond to our datasets and \textbf{HumanEval}~\citep{chen2021evaluating}. Bits per byte (BPB) is the primary metric. We define the mean BPB value of the seven benchmarks as \textbf{Code Mean}. 

\subsection{Main results}
\label{mainres}

\subsubsection{Lev outperforms existing diversity-based selection methods}
\label{efflev}

We conduct a systematic comparison against existing diversity-based methods across three dimensions: direct dataset diversity evaluation, downstream language performance on CommonCrawl subsets and downstream code performance on StarCoderData subsets. For diversity measurements, we adopt seven diversity metrics from~\cite{yang2025measuring} to evaluate dataset diversity. Specifically, CosDist, KNN, and CluIne are distance-based metrics that emphasize sample-level uniqueness or dispersion. Entropy, Vendi score, and LogDet are spectrum-based metrics that reflect geometric diversity through the eigenstructure of the Gram or covariance matrix. NovelSum integrates both perspectives for a more comprehensive diversity characterization. We normalize the raw values of each metric via affine transformations that map the observed minimum and maximum to 0.3 and 1.0 respectively. The normalized values are shown in Figure~\ref{radar}. We provide full definitions of metrics in Appendix~\ref{div_define} and full results in Appendix \ref{table_div_detailed}.

\begin{wrapfigure}{r}{0.33\linewidth}
  \vspace{-8pt}
  \centering
  \includegraphics[width=\linewidth]{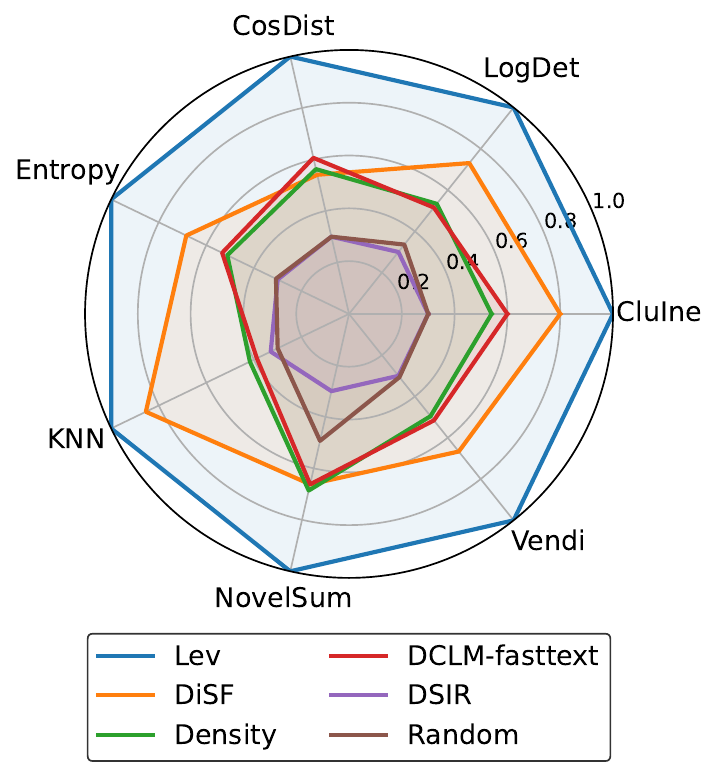}
  \caption{Diversity comparison.}
  \label{radar}
\end{wrapfigure}

\textbf{Diversity measurement of selected subsets.} Figure~\ref{radar} shows that Lev achieves the best performance across all diversity metrics on Problem~\eqref{eq:problem}, surpassing the strong diversity-oriented baseline, DiSF. Specifically, Lev improves the Vendi score by 21.3\% over Random (168.06 vs. 203.81) and by 9.2\% over DiSF (186.63 vs. 203.81). In contrast, DSIR collapses into two specific domains (Wikipedia and books), leading to severe diversity degradation. Density, being a coverage-oriented method, does not directly optimize for geometric diversity and thus trails behind DiSF and Lev. Finally, Lev outperforms DiSF because it better captures the per-sample contribution to subspace span, eventually yielding selected datasets of substantially higher diversity. We provide a comparative theoretical analysis over the optimization objectives of Lev and DiSF in Appendix \ref{divopt}, which shows that Lev obtains better diversity results because it is more sensitive in detecting diverse samples.

\begin{table}[t]
\renewcommand{\arraystretch}{1.25}
\caption{Downstream performance on CommonCrawl subset selection. Higher is better.\\ Accuracy(\%) is reported with subscripts indicating the ranking. (Best in bold. Second best in italics.)}
\label{cc_perf}
\begin{center}
\begin{small}
\begin{tabular}{l|ccccccc|c|c}
\noalign{\vskip 0.5pt}
\noalign{\hrule height 0.75pt}
\noalign{\vskip 0.5pt}
\multirow{2}{*}{Method} & \multicolumn{8}{c|}{50\% Selection Ratio} & {40\% Ratio}\\
\cline{2-9}
 & ARC-e & ARC-c & WG & HS & PIQA & SIQA & CSQA & Mean \textuparrow & Mean \textuparrow\\
\hline
Random & 48.50 & 27.82 & 52.09 & 46.50 & \textbf{69.40} & \textbf{46.30} & \textbf{43.00} & 47.66$ _{(3)}$ & 47.66$ _{(4)}$ \\
DSIR & \textbf{53.50} & 27.73 & \textbf{53.04} & 42.00 & 65.20 & 44.30 & 42.26 & 46.86$ _{(5)}$ & 44.84$ _{(5)}$ \\
Density & 51.70 & 27.99 & 50.75 & \textbf{48.20} & 66.40 & 45.00 & 42.51 & 47.51$ _{(4)}$ & 47.98$ _{(3)}$ \\
DiSF & \textit{52.90} & \textit{28.16} & \textit{52.49} & 46.40 & 67.80 & 44.50 & \textit{42.59} & \textit{47.83}$ _{(2)}$ & \textit{48.35}$ _{(2)}$ \\
Lev (ours) & 52.70 & \textbf{29.35} & 52.33 & \textit{47.10} & \textit{68.30} & \textit{45.10} & 42.34 & \textbf{48.17}$ _{(1)}$ & \textbf{48.91}$ _{(1)}$ \\
\noalign{\vskip 0.5pt}
\noalign{\hrule height 0.75pt}
\noalign{\vskip 0.5pt}
\end{tabular}
\end{small}
\end{center}
\end{table}

\textbf{Downstream performance for CommonCrawl subset selection.} 
We evaluate models trained on CommonCrawl subsets selected by different methods across seven language understanding tasks. Random and DSIR serve as baselines. Table~\ref{cc_perf} shows that Lev consistently achieves the best performance among all diversity-based methods, yielding an average improvement of 1.31\% over DSIR at a 50\% selection ratio. DSIR peaks on ARC-e and WinoGrande but bottoms out on ARC-c and HellaSwag, exhibiting the performance imbalance caused by dimensional collapse. In contrast, Lev achieves the best or second-best performance on ARC-c, HellaSwag, PIQA, and SIQA, demonstrating that diversity enhancement fosters more balanced knowledge extraction from CC web data and yields consistently strong performance across tasks.

\textbf{Downstream performance for StarCoderData subset selection.} 
We show that for code data, where robust quality heuristics are difficult to define, Lev is effective at curating high-quality code datasets in an unsupervised manner. In fact, our findings reveal that diversification itself improves the utility of code data. To validate this, we prepare 60B-token pretraining corpora consisting of 70\% CC data and 30\% StarCoderData. We apply different diversity methods to select code subsets at 40\% and 50\% selection ratios, and combine the selected subsets with two types of CC data: Lev-sampled and quality-filtered CC subsets (DCLM-fastText serves as the quality scorer). This yields four data combinations. We then evaluate BPB of the resulting models on code-related tasks. The results are shown in Figure~\ref{fig:code_all} (full results in Appendix~\ref{table_code_detailed}). We draw the following conclusions:

\begin{figure}[t]
\centering
\begin{subfigure}{0.49\textwidth}
    \centering
    \includegraphics[width=\linewidth]{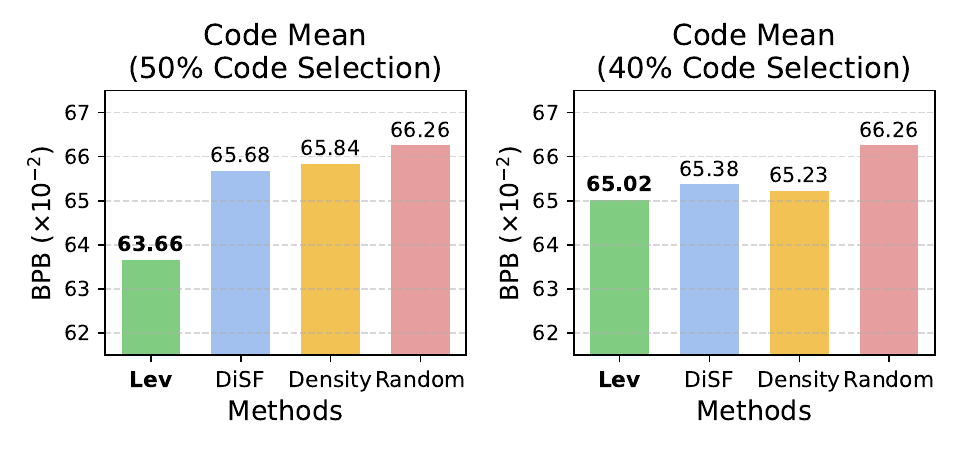}
    \caption{Combined with quality-filtered CC data.}
    \label{code_a}
\end{subfigure}
\hfill
\begin{subfigure}{0.49\textwidth}
    \centering
    \includegraphics[width=\linewidth]{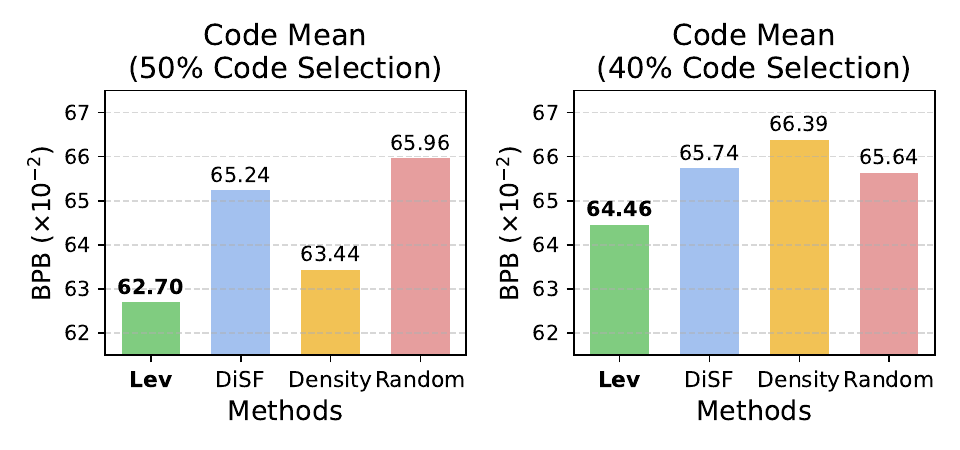}
    \caption{Combined with Lev-sampled CC data.}
    \label{code_bb}
\end{subfigure}
\caption{Code performance of different diversity-based methods. Lower is better.}
\label{fig:code_all}
\end{figure}

\begin{enumerate}[leftmargin=1.4em,itemsep=0pt,topsep=0pt]
    \item \textbf{Diversified code data improves code performance.} Almost all diversity-based methods applied to code data selection improve downstream performance on code tasks, as reflected by lower BPB values relative to the uncurated baseline (Random) in Figure~\ref{fig:code_all}.
    \item \textbf{Lev achieves the best code performance among diversity-based methods.} Lev consistently attains the lowest code BPB across all four settings in Figure~\ref{fig:code_all}. Specifically, at a 50\% code selection ratio with quality-filtered web data (Figure~\ref{code_a}), Lev reduces average code BPB by 3.92\% over Random (from 66.26 to 63.66) and by 3.08\% over DiSF (from 65.68 to 63.66). With Lev-sampled web data (Figure~\ref{code_bb}), the improvement over Random increases to 4.94\% (from 65.96 to 62.70). Results at the 40\% selection ratio corroborate this finding.
\end{enumerate}

\subsubsection{Revisiting dimensional collapse: quality filtering demands diversity}
\label{crossdeg}

\begin{table}[ht]
\renewcommand{\arraystretch}{1.35}
\centering
\vspace{-10pt}
\caption{Comparison between DCLM-fastText (Qual) and Lev under two data recipe.\\ BPB($\times 10^{-2}$) for Code Mean and Accuracy(\%) for Base Mean. Subscripts indicate the ranking.}
\begin{subtable}{0.48\linewidth}
\centering
\caption{Recipe: web only}
\begin{small}
\begin{tabular}{l|c|c}
\noalign{\vskip 0.5pt}
\noalign{\hrule height 0.75pt}
\noalign{\vskip 0.5pt}
Web Data & Code Mean \textdownarrow & Base Mean \textuparrow \\
\hline
Lev (top 40\%) & \textbf{256.43}$ _{(1)}$ & 48.91$ _{(2)}$ \\
Lev (top 50\%) & \textbf{257.35}$ _{(2)}$ & 48.17$ _{(3)}$ \\
Qual (top 40\%) & 266.77$ _{(3)}$ & \textbf{50.21}$ _{(1)}$ \\
Random & 273.22$ _{(4)}$ & 47.66$ _{(4)}$ \\
\noalign{\vskip 0.5pt}
\noalign{\hrule height 0.75pt}
\noalign{\vskip 0.5pt}
\end{tabular}
\end{small}
\label{tab:b}
\end{subtable}
\hfill
\begin{subtable}{0.48\linewidth}
\centering
\caption{Recipe: 70\% web + 30\% Lev curated code data}
\begin{small}
\begin{tabular}{l|c|c}
\noalign{\vskip 0.5pt}
\noalign{\hrule height 0.75pt}
\noalign{\vskip 0.5pt}
Web Data & Code Mean \textdownarrow & Base Mean \textuparrow \\
\hline
Lev (top 40\%) & \textbf{62.48}$ _{(1)}$ & 49.90$ _{(2)}$ \\
Lev (top 50\%) & \textbf{62.70}$ _{(2)}$ & 49.17$ _{(4)}$ \\
Qual (top 40\%) & 63.66$ _{(3)}$ & \textbf{51.53}$ _{(1)}$ \\
Random & 64.46$ _{(4)}$ & 49.28$ _{(3)}$ \\
\noalign{\vskip 0.5pt}
\noalign{\hrule height 0.75pt}
\noalign{\vskip 0.5pt}
\end{tabular}
\end{small}
\label{tab:a}
\end{subtable}
\label{comp2}
\end{table}

Beyond diversity-based methods, quality filtering has long dominated pretraining data curation, owing to its stronger performance on downstream commonsense tasks. However, we overturn this conventional wisdom by uncovering a previously unknown phenomenon: quality filtering can induce cross-domain dimensional collapse, which leads to persistent suboptimality on domain-specific modeling capability. 

\begin{wrapfigure}{r}{0.36\linewidth}
  \centering
  \includegraphics[width=\linewidth]{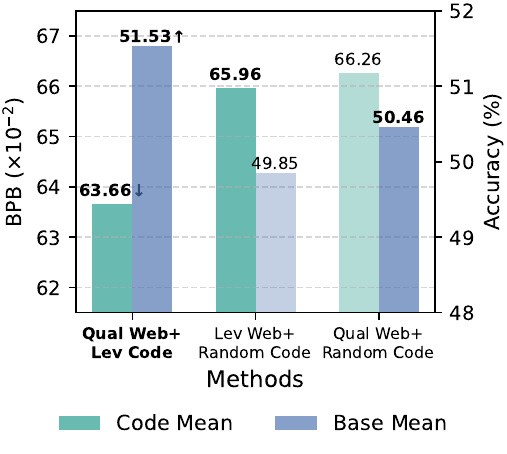}
  \vspace{-10pt}
  \caption{Quality-filtered CC + Lev-selected code data achieves the best performance.}
  \label{mitigate}
  \vspace{-15pt}
\end{wrapfigure}

Using DCLM-fastText as a case study, we find that it likely overweights high-quality text comprehension data relative to code-related data in CC, yielding a text--code imbalance in the filtered data. To demonstrate this, we first pretrain models separately on quality-filtered CC data and Lev-selected CC data. Table~\ref{tab:b} shows that models trained on quality-filtered CC data exhibit significantly higher code BPB values than those trained on Lev-selected data, suggesting that the former contains less code-related content. We then augment both curated CC subsets with dedicated code data sampled from StarCoderData using Lev, and train additional models on these mixtures. Nevertheless, Table~\ref{tab:a} shows that the code performance gap persists. We repeat this experiment with two additional code datasets, and observe consistent results, which are fully reported in Appendix~\ref{infapp}. We also conduct further investigation in Appendix~\ref{appe2}, demonstrating that the inadequate recall of code data is a common limitation among open-source data scorers. This indicates that quality filtering necessitates a broader consideration of diversity.

Applying Lev to code data partially mitigates the code inferiority of quality-filtered web data, resulting in the best overall performance. Figure~\ref{mitigate} shows that combining quality-filtered CC data with Lev-selected code data yields a model that achieves lower code BPB (63.66 vs. 65.96) than the one trained on Lev-selected CC mixed with unfiltered code, while also improving text comprehension accuracy (51.53\% vs. 50.46\%). However, we emphasize that deeply integrating diversity and quality objectives to sample a more holistic CC subset remains a critical direction for future research. Concurrent work has explored effective methods along these lines~\citep{liu2025quadmix,fan2025joint}. Ultimately, our work highlights the indispensable role of dataset diversity in LLM pretraining data curation.

\subsubsection{Lev enables efficient large-scale diversity selection}

While Lev achieves strong diversification, it also proves more efficient than DiSF in both theoretical complexity and empirical runtime, enabling larger-scale data selection.

\begin{table}[ht]
\renewcommand{\arraystretch}{1.25}
\vspace{-12pt}
\caption{Runtime comparison in two real scenarios: web and code selection.}
\label{tab:eff_comp}
\begin{center}
\begin{small}
\begin{tabular}{c|cc|c|cc|c}
\noalign{\vskip 0.5pt}
\noalign{\hrule height 0.75pt}
\noalign{\vskip 0.5pt}
\multirow{2}{*}{Step size} & \multicolumn{3}{c|}{CommonCrawl} & \multicolumn{3}{c}{StarCoderData}\\
\cline{2-7}
 & Lev & DiSF & Speedup & Lev & DiSF & Speedup \\
\hline
1 & 3.03h & 217.28h & 72$\times$ & 1.03h & 67.45h & 65$\times$ \\
2 & 1.67h & 108.48h & 65$\times$ & 34min & 33.68h & 59$\times$ \\
4 & 55min & 56.12h & 61$\times$ & 19min & 17.05h & 54$\times$ \\
8 & 35min & 26.88h & 46$\times$ & 11min & 8.58h & 44$\times$ \\
\noalign{\vskip 0.5pt}
\noalign{\hrule height 0.75pt}
\noalign{\vskip 0.5pt}
\end{tabular}
\end{small}
\end{center}
\vspace{-10pt}
\end{table}

\begin{wrapfigure}{r}{0.35\linewidth}
  \centering
  \includegraphics[width=\linewidth]{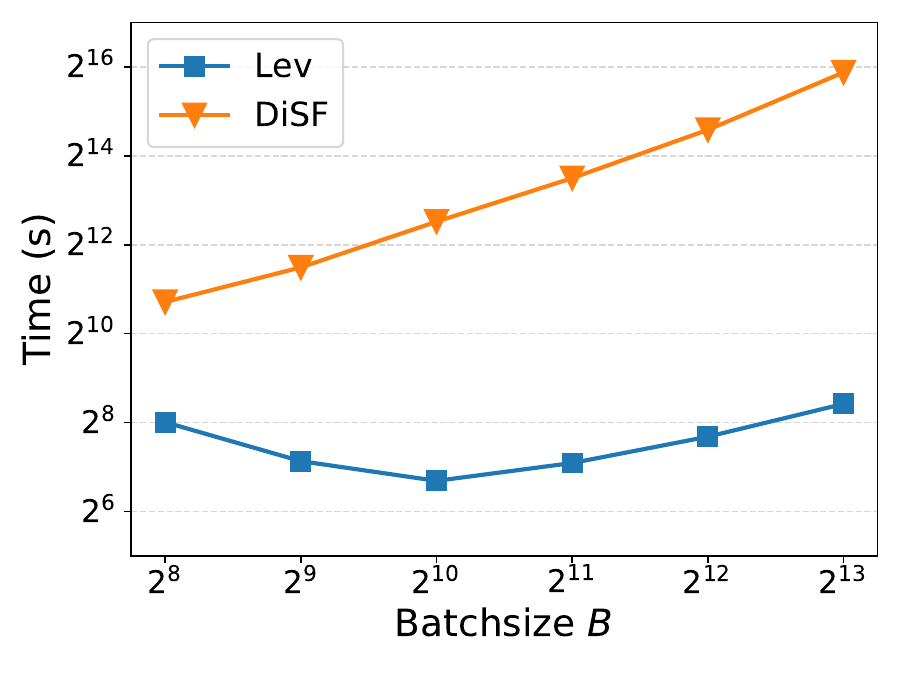}
  \caption{Runtime comparison between Lev and DiSF with respect to batch size $B$.}
  \label{batchsize}
  \vspace{-10pt}
\end{wrapfigure}

\textbf{Theoretical Analysis.} 
For Lev, Algorithm~\ref{alg:leverage_sampling} processes the corpus in $M = D/B$ batches. Within each batch, it performs $m = \alpha B/b$ iterations. The dominant computational cost per iteration is $O(Bd^2 + d^3)$ for computing leverage scores, comprising: (1) $O(d^3)$ for Cholesky decomposition, (2) $O(Bd^2)$ for batch-wise leverage score computation, and (3) $O(bd^2)$ for Gram matrix updates. To conclude, Lev has an overall time complexity of $O(Mm(Bd^2 + d^3)) = O(D B d^2)$ when $B > d$. 

In contrast, DiSF's equivalent score function is $f(z_i) = e^{-\|C(\mathcal{S} \cup \{z_i\})\|_F}$, where $C(\mathcal{S})$ is the normalized feature covariance matrix with respect to the selected data \(\mathcal{S}\). This involves $O(\alpha B^2d^2)$ complexity per iteration, as computing $f(z_i)$ for $B$ candidates requires $O(\alpha Bd^2)$ operations each. Consequently, DiSF's overall complexity is $O(Mm\alpha B^2d^2) = O(D B^2 d^2)$. 
We provide a more elaborated analysis in Appendix \ref{detailed_time}.

\textbf{Runtime Comparison.} 
The former analysis indicates that Lev eliminates one order of $B$ compared to DiSF. To validate this, we measure the runtime of data selection with respect to the batch size $B$. We select 40\% of the documents from a CommonCrawl subset using 8 NVIDIA RTX 5090 GPUs that process 8 files in parallel. Figure~\ref{batchsize} shows that Lev is faster than DiSF by several orders of magnitude overall, and the gap widens as $B$ grows. We then evaluate both algorithms in two realistic settings: CommonCrawl and StarCoderData subset selection. We adopt a batch size of 1024, which is the best-practice setting for DiSF in~\citep{fan2025combatting}, and measure the runtimes across different step sizes. The greedy selection for submodular function optimization in DiSF corresponds to the step size of 1. 
Table~\ref{tab:eff_comp} shows that with a step size of 1, Lev processes a CommonCrawl dump in only about 3.03 hours, whereas DiSF requires 217.28 hours, \textit{confirming a $72\times$ speedup}. For Lev, Figure~\ref{batchsize} exhibits an anomalous monotonic decrease when $B$ is small. This occurs because leverage sampling is extremely efficient in this regime, shifting the bottleneck to other parts of the pipeline. Overall, Lev exhibits superior scalability compared to DiSF.

\subsection{Ablation Study}
\label{subsec:ablation}

\begin{figure}[t]
\centering

\begin{subfigure}{0.32\textwidth}
    \centering
    \includegraphics[width=\linewidth]{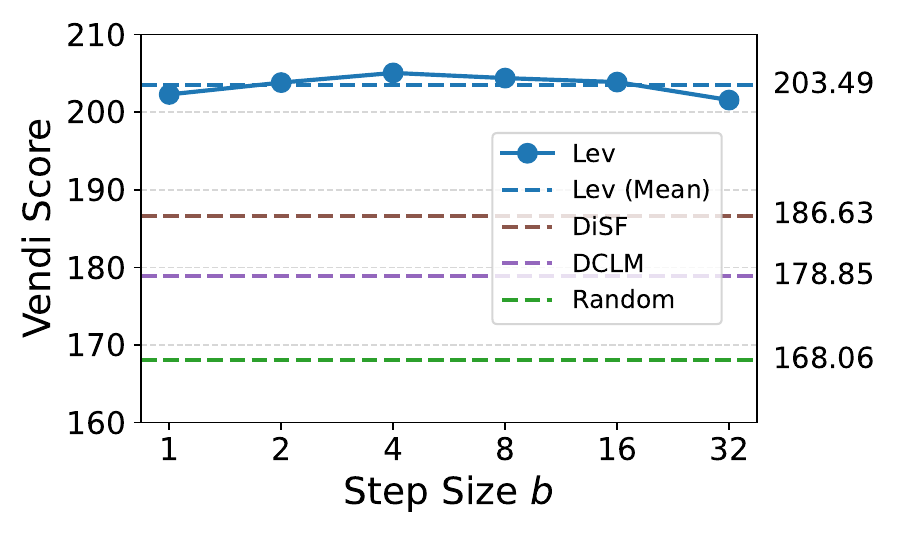}
    \caption{Vendi score vs. Step size.}
    \label{sensea}
\end{subfigure}
\hfill
\begin{subfigure}{0.32\textwidth}
    \centering
    \includegraphics[width=\linewidth]{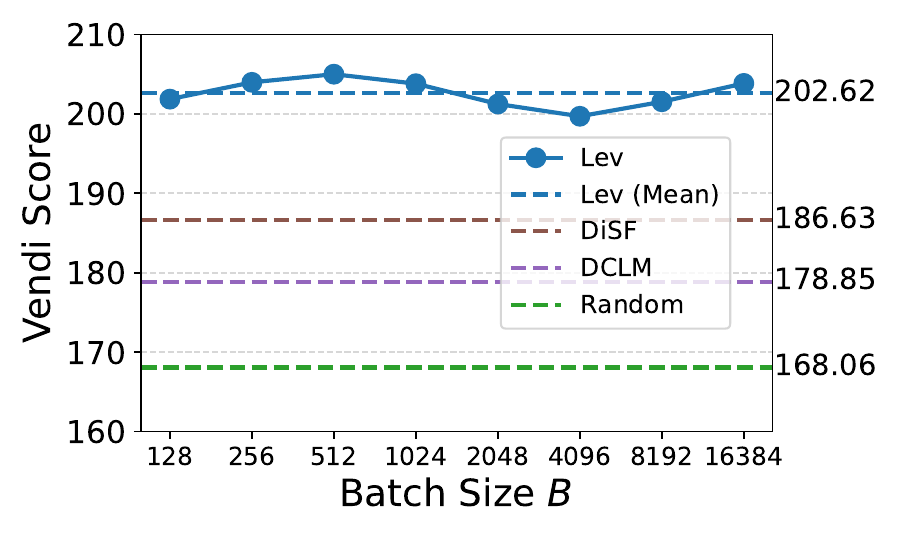}
    \caption{Vendi score vs. Batch size.}
    \label{senseb}
\end{subfigure}
\hfill
\begin{subfigure}{0.32\textwidth}
    \centering
    \includegraphics[width=\linewidth]{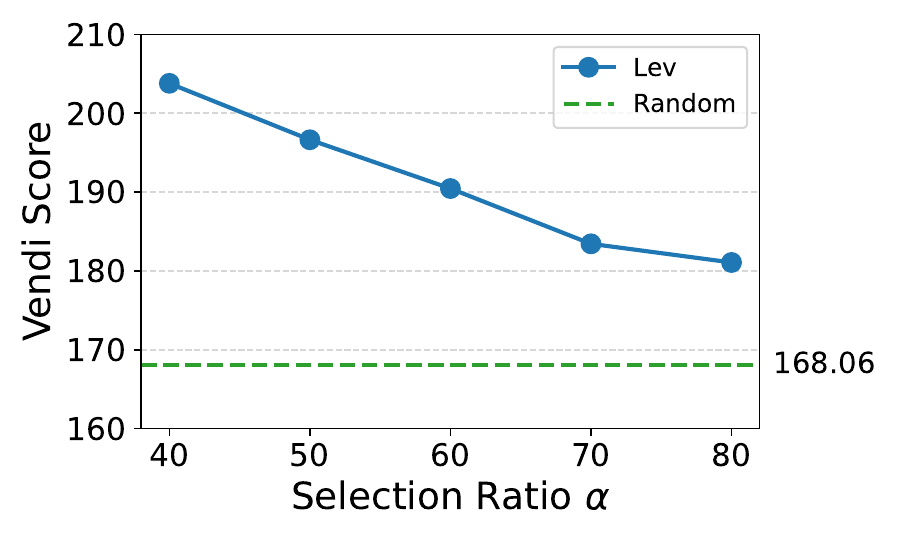}
    \caption{Vendi score vs. Selection ratio.}
    \label{sensec}
\end{subfigure}

\vspace{0.5em}

\begin{subfigure}{0.32\textwidth}
    \centering
    \includegraphics[width=\linewidth]{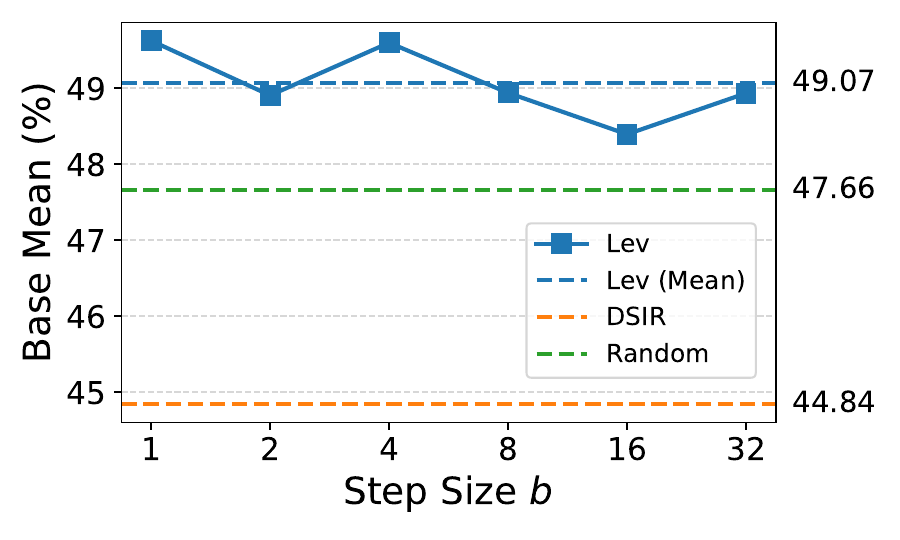}
    \caption{Base mean vs. Step size.}
    \label{sensed}
\end{subfigure}
\hfill
\begin{subfigure}{0.32\textwidth}
    \centering
    \includegraphics[width=\linewidth]{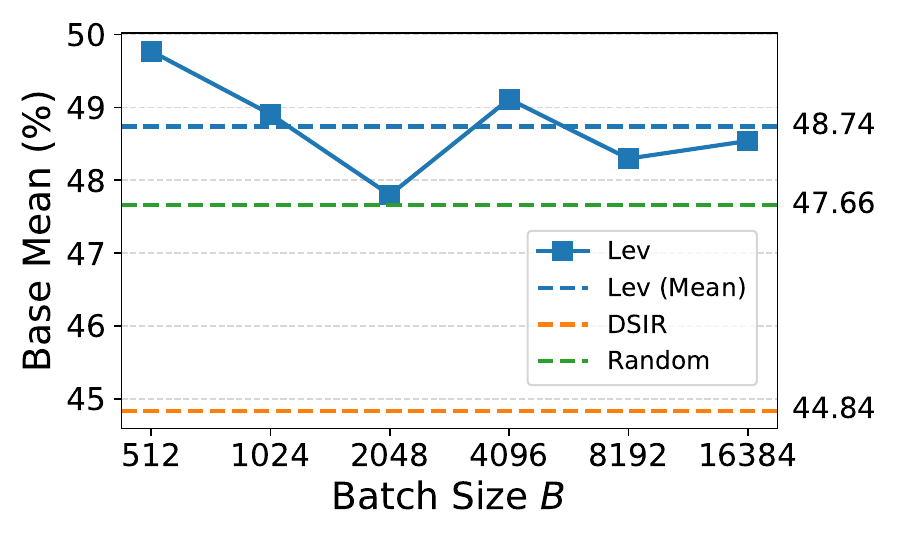}
    \caption{Base mean vs. Batch size.}
    \label{sensee}
\end{subfigure}
\hfill
\begin{subfigure}{0.32\textwidth}
    \centering
    \includegraphics[width=\linewidth]{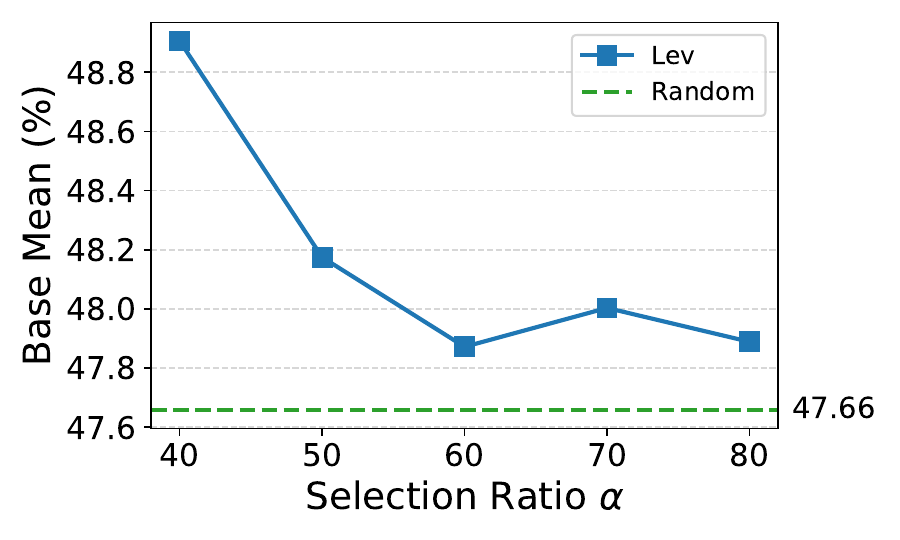}
    \caption{Base mean vs. Selection ratio.}
    \label{sensef}
\end{subfigure}

\caption{Sensitivity to hyperparameters.}
\label{sense}
\end{figure}

We investigate the sensitivity of Algorithm~\ref{alg:leverage_sampling} to key hyperparameters: step size, batch size, and selection ratio. To isolate effects from cross-domain interference, we consider only CC web data at a 40\% selection ratio for evaluating step size and batch size. We show that Lev is robust to variations in step size and batch size, and remains effective across a range of selection ratios for both CC and code.

\begin{wrapfigure}{r}{0.4\linewidth}
  \vspace{-5pt}
  \centering
  \includegraphics[width=\linewidth]{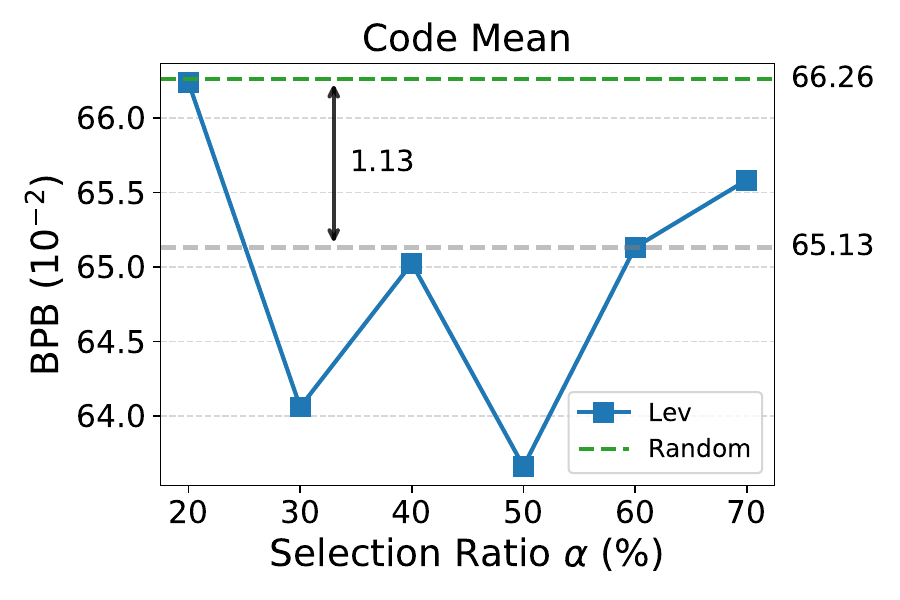}
  \caption{Code mean vs. Code selection ratio. Lower is better.}
  \vspace{-5pt}
  \label{code_ratio}
\end{wrapfigure}

\textbf{Step size \(b\) is insensitive.} A larger step size implies that more samples are selected in a single update. Intuitively, this might appear to compromise diversity. Contrary to this intuition, Figure~\ref{sensea} suggests otherwise. Because the raw data is shuffled, the $b$ samples selected at each step are highly likely to be distinct. Thus, a larger step size preserves data diversity. This is further corroborated by Figure~\ref{sensed}, which shows that downstream scores fluctuate around the mean without exhibiting significant monotonicity, confirming that step size is an insensitive parameter in Algorithm~\ref{alg:leverage_sampling}. Empirically, a larger step size is preferred because of its substantial computational speedups (Table~\ref{tab:eff_comp}).

\textbf{Batch size \(B\) is insensitive.} Intuitively, a larger batch size might appear to ensure higher data diversity, as the reference data pool $X$ covers more selected samples during each iteration. However, Figures~\ref{senseb} and~\ref{sensee} contradict this expectation. Both diversity and downstream scores fluctuate around their respective means without exhibiting a significant monotonic trend with respect to $B$, indicating that the batch size is also an insensitive parameter in our algorithm. Since a larger $B$ incurs greater computational overhead without yielding corresponding diversity gains, we select $B=1024$ to balance efficiency and performance. As shown in Table~\ref{div_metric}, this setting already exhibits a high degree of diversity and achieves a considerable advantage over other baselines.

\textbf{Selection ratio \(\alpha\).} Table~\ref{sensec} shows that web data diversity decreases as the selection ratio increases, leading to a corresponding drop in downstream performance in Table~\ref{sensef}. However, Lev consistently outperforms the Random baseline even when \(\alpha\) exceeds 50\%, underscoring the effectiveness of data diversification. Particularly for code data selection shown in Figure~\ref{code_ratio}, we observe that Lev is consistently effective within a wide range of selection ratios (e.g., 30\% to 60\%), yielding a stable reduction of over $1.13\times 10^{-2}$ in bits-per-byte compared to the baseline.

\section{Conclusion}
\label{sec:conclusion}

We present Leverage Score Sampling (Lev), a scalable diversity-oriented method that efficiently expands the representation subspace by iteratively selecting high-leverage samples. Empirically, Lev outperforms existing diversification baselines across diversity metrics, downstream task performance, and runtime efficiency while simultaneously achieving a $72\times$ speedup over DiSF. Our analysis further reveals a collateral {cross-domain dimensional collapse} induced by quality filtering, exemplified by the systematic exclusion of code-related content from CommonCrawl. These findings underscore the essential complementary role of diversity-aware selection to quality filtering and inspire tighter integration of diversity and quality objectives in building more robust data curation pipelines.

\begin{ack}
    This work is sponsored by the National Natural Science Foundation of China (No. 22494644, No. 62561160152). Any opinions, findings, conclusions, or recommendations expressed in this material are those of the author(s) and do not necessarily reflect the views of the funding agencies. This work was done during the first author's internship at Wizard Intelligence Learning Lab.

\end{ack}


\bibliographystyle{plain}
\bibliography{references}

@inproceedings{stasaski2022semantic,
  title={Semantic diversity in dialogue with natural language inference},
  author={Stasaski, Katherine and Hearst, Marti A},
  booktitle={Proceedings of the 2022 Conference of the North American Chapter of the Association for Computational Linguistics: Human Language Technologies},
  pages={85--98},
  year={2022}
}

@InProceedings{pasarkar2023cousins,
  title = 	 {Cousins Of The {V}endi Score: A Family Of Similarity-Based Diversity Metrics For Science And Machine Learning},
  author =       {Pasarkar, Amey P. and Dieng, Adji Bousso},
  booktitle = 	 {Proceedings of The 27th International Conference on Artificial Intelligence and Statistics},
  year = 	 {2024},
}

@inproceedings{du2019boosting,
  title={Boosting dialog response generation},
  author={Du, Wenchao and Black, Alan W},
  booktitle={Proceedings of the 57th Annual Meeting of the Association for Computational Linguistics},
  pages={38--43},
  year={2019}
}

@article{zhang2025qwen3,
  title={Qwen3 embedding: Advancing text embedding and reranking through foundation models},
  author={Zhang, Yanzhao and Li, Mingxin and Long, Dingkun and Zhang, Xin and Lin, Huan and Yang, Baosong and Xie, Pengjun and Yang, An and Liu, Dayiheng and Lin, Junyang and others},
  journal={arXiv preprint arXiv:2506.05176},
  year={2025}
}

@inproceedings{
  hao2025bliss,
  title={{BLISS}: A Lightweight Bilevel Influence Scoring Method for Data Selection in Language Model Pretraining},
  author={Jie Hao and Rui Yu and Wei Zhang and Huixia Wang and Jie Xu and Mingrui Liu},
  booktitle={Forty-third International Conference on Machine Learning},
  year={2026}
}

@inproceedings{fan2025combatting,
  title={Combatting dimensional collapse in LLM pre-training data via submodular file selection},
  author={Fan, Ziqing and Du, Siyuan and Hu, Shengchao and Wang, Pingjie and Shen, Li and Zhang, Ya and Tao, Dacheng and Wang, Yanfeng},
  booktitle={The Thirteenth International Conference on Learning Representations},
  year={2025}
}

@article{xie2023data,
  title={Data selection for language models via importance resampling},
  author={Xie, Sang Michael and Santurkar, Shibani and Ma, Tengyu and Liang, Percy S},
  journal={Advances in Neural Information Processing Systems},
  volume={36},
  pages={34201--34227},
  year={2023}
}

@inproceedings{
  gu2024data,
  title={Data Selection via Optimal Control for Language Models},
  author={Yuxian Gu and Li Dong and Hongning Wang and Yaru Hao and Qingxiu Dong and Furu Wei and Minlie Huang},
  booktitle={The Thirteenth International Conference on Learning Representations},
  year={2025},
  url={https://openreview.net/forum?id=dhAL5fy8wS}
}

@article{li2024datacomp,
  title={Datacomp-lm: In search of the next generation of training sets for language models},
  author={Li, Jeffrey and Fang, Alex and Smyrnis, Georgios and Ivgi, Maor and Jordan, Matt and Gadre, Samir Yitzhak and Bansal, Hritik and Guha, Etash and Keh, Sedrick Scott and Arora, Kushal and others},
  journal={Advances in Neural Information Processing Systems},
  volume={37},
  pages={14200--14282},
  year={2024}
}

@inproceedings{peng2025dataman,
  title={Dataman: Data manager for pre-training large language models},
  author={Peng, Ru and Yang, Kexin and Zeng, Yawen and Lin, Junyang and Liu, Dayiheng and Zhao, Junbo},
  booktitle={International Conference on Learning Representations},
  volume={2025},
  pages={9747--9799},
  year={2025}
}

@inproceedings{wang2024diversity,
  title={Diversity Measurement and Subset Selection for Instruction Tuning Datasets},
  author={Peiqi Wang and Yikang Shen and Zhen Guo and Matthew Stallone and Yoon Kim and Polina Golland and Rameswar Panda},
  booktitle={ICLR 2025 Workshop on Navigating and Addressing Data Problems for Foundation Models},
  year={2025}
}

@article{xie2023doremi,
  title={Doremi: Optimizing data mixtures speeds up language model pretraining},
  author={Xie, Sang Michael and Pham, Hieu and Dong, Xuanyi and Du, Nan and Liu, Hanxiao and Lu, Yifeng and Liang, Percy S and Le, Quoc V and Ma, Tengyu and Yu, Adams Wei},
  journal={Advances in Neural Information Processing Systems},
  volume={36},
  pages={69798--69818},
  year={2023}
}

@inproceedings{
  fan2025joint,
  title={Joint Selection for Large-Scale Pre-Training Data via Policy Gradient-based Mask Learning},
  author={Ziqing Fan and Yuqiao Xian and Yan Sun and Ke Shen and Li Shen},
  booktitle={The Fourteenth International Conference on Learning Representations},
  year={2026},
  url={https://openreview.net/forum?id=fs2uDib85s}
}

@article{yu2024mates,
  title={Mates: Model-aware data selection for efficient pretraining with data influence models},
  author={Yu, Zichun and Das, Spandan and Xiong, Chenyan},
  journal={Advances in Neural Information Processing Systems},
  volume={37},
  pages={108735--108759},
  year={2024}
}

@inproceedings{yang2025measuring,
  title={Measuring data diversity for instruction tuning: A systematic analysis and a reliable metric},
  author={Yang, Yuming and Nan, Yang and Ye, Junjie and Dou, Shihan and Wang, Xiao and Li, Shuo and Lv, Huijie and Gui, Tao and Zhang, Qi and Huang, Xuan-Jing},
  booktitle={Proceedings of the 63rd Annual Meeting of the Association for Computational Linguistics (Volume 1: Long Papers)},
  pages={18530--18549},
  year={2025}
}

@InProceedings{wettig2025organize,
  title = 	 {Organize the Web: Constructing Domains Enhances Pre-Training Data Curation},
  author =       {Wettig, Alexander and Lo, Kyle and Min, Sewon and Hajishirzi, Hannaneh and Chen, Danqi and Soldaini, Luca},
  booktitle = 	 {Proceedings of the 42nd International Conference on Machine Learning},
  year = 	 {2025}
}

@inproceedings{zhang2025preference,
  title={Preference curriculum: Llms should always be pretrained on their preferred data},
  author={Zhang, Xuemiao and Liangyu, Xu and Duan, Feiyu and Zhou, Yongwei and Wang, Sirui and Weng, Rongxiang and Wang, Jingang and Cai, Xunliang},
  booktitle={Findings of the Association for Computational Linguistics: ACL 2025},
  pages={21181--21198},
  year={2025}
}

@article{liu2025quadmix,
  title={Quadmix: Quality-diversity balanced data selection for efficient llm pretraining},
  author={Liu, Fengze and Zhou, Weidong and Liu, Binbin and Yu, Zhimiao and Zhang, Yifan and Lin, Haobin and Yu, Yifeng and Zhang, Bingni and Zhou, Xiaohuan and Wang, Taifeng and others},
  journal={arXiv preprint arXiv:2504.16511},
  year={2025}
}

@inproceedings{nguyen2025rewire,
  title={Recycling the Web: A Method to Enhance Pre-training Data Quality and Quantity for Language Models},
  author={Thao Nguyen and Yang Li and Olga Golovneva and Luke Zettlemoyer and Sewoong Oh and Ludwig Schmidt and Xian Li},
  booktitle={Second Conference on Language Modeling},
  year={2025}
}

@inproceedings{yu2025repro,
  title={RePro: Training Language Models to Faithfully Recycle the Web for Pretraining},
  author={Zichun Yu and Chenyan Xiong},
  booktitle={Forty-third International Conference on Machine Learning},
  year={2026}
}

@article{penedo2024fineweb,
  title={The fineweb datasets: Decanting the web for the finest text data at scale},
  author={Penedo, Guilherme and Kydl{\'\i}{\v{c}}ek, Hynek and Lozhkov, Anton and Mitchell, Margaret and Raffel, Colin A and Von Werra, Leandro and Wolf, Thomas and others},
  journal={Advances in Neural Information Processing Systems},
  volume={37},
  pages={30811--30849},
  year={2024}
}

@inproceedings{gu2025olmes,
  title={Olmes: A standard for language model evaluations},
  author={Gu, Yuling and Tafjord, Oyvind and Kuehl, Bailey and Haddad, Dany and Dodge, Jesse and Hajishirzi, Hannaneh},
  booktitle={Findings of the Association for Computational Linguistics: NAACL 2025},
  pages={5005--5033},
  year={2025}
}

@inproceedings{zellers2019hellaswag,
  title={Hellaswag: Can a machine really finish your sentence?},
  author={Zellers, Rowan and Holtzman, Ari and Bisk, Yonatan and Farhadi, Ali and Choi, Yejin},
  booktitle={Proceedings of the 57th annual meeting of the association for computational linguistics},
  pages={4791--4800},
  year={2019}
}

@inproceedings{bisk2020piqa,
  title={Piqa: Reasoning about physical commonsense in natural language},
  author={Bisk, Yonatan and Zellers, Rowan and Gao, Jianfeng and Choi, Yejin and others},
  booktitle={Proceedings of the AAAI conference on artificial intelligence},
  volume={34},
  number={05},
  pages={7432--7439},
  year={2020}
}

@inproceedings{sakaguchi2020winogrande,
  title={Winogrande: An adversarial winograd schema challenge at scale},
  author={Sakaguchi, Keisuke and Le Bras, Ronan and Bhagavatula, Chandra and Choi, Yejin},
  booktitle={Proceedings of the AAAI Conference on Artificial Intelligence},
  volume={34},
  number={05},
  pages={8732--8740},
  year={2020}
}

@inproceedings{talmor2019commonsenseqa,
  title={Commonsenseqa: A question answering challenge targeting commonsense knowledge},
  author={Talmor, Alon and Herzig, Jonathan and Lourie, Nicholas and Berant, Jonathan},
  booktitle={Proceedings of the 2019 Conference of the North American Chapter of the Association for Computational Linguistics: Human Language Technologies, Volume 1 (Long and Short Papers)},
  pages={4149--4158},
  year={2019}
}

@inproceedings{sap2019socialiqa,
  title={Social IQa: Commonsense reasoning about social interactions},
  author={Sap, Maarten and Rashkin, Hannah and Chen, Derek and Le Bras, Ronan and Choi, Yejin},
  booktitle={Proceedings of the 2019 conference on empirical methods in natural language processing and the 9th international joint conference on natural language processing (EMNLP-IJCNLP)},
  pages={4463--4473},
  year={2019}
}

@article{clark2018think,
  title={Think you have solved question answering? try arc, the ai2 reasoning challenge},
  author={Clark, Peter and Cowhey, Isaac and Etzioni, Oren and Khot, Tushar and Sabharwal, Ashish and Schoenick, Carissa and Tafjord, Oyvind},
  journal={arXiv preprint arXiv:1803.05457},
  year={2018}
}

@article{chen2021evaluating,
  title={Evaluating large language models trained on code},
  author={Chen, Mark},
  journal={arXiv preprint arXiv:2107.03374},
  year={2021}
}

@article{penedo2023refinedweb,
  title={The refinedweb dataset for falcon llm: Outperforming curated corpora with web data only},
  author={Penedo, Guilherme and Malartic, Quentin and Hesslow, Daniel and Cojocaru, Ruxandra and Alobeidli, Hamza and Cappelli, Alessandro and Pannier, Baptiste and Almazrouei, Ebtesam and Launay, Julien},
  journal={Advances in Neural Information Processing Systems},
  volume={36},
  pages={79155--79172},
  year={2023}
}

@article{
  li2023starcoder,
  title={StarCoder: may the source be with you!},
  author={Li, Raymond and Allal, Loubna Ben and Zi, Yangtian and Muennighoff, Niklas and Kocetkov, Denis and Mou, Chenghao and Marone, Marc and Akiki, Christopher and Li, Jia and Chim, Jenny and others},
  journal={Transactions on Machine Learning Research},
  issn={2835-8856},
  year={2023}
}

@misc{cassano2022multiplescalableextensibleapproach,
      title={MultiPL-E: A Scalable and Extensible Approach to Benchmarking Neural Code Generation}, 
      author={Federico Cassano and John Gouwar and Daniel Nguyen and Sydney Nguyen and Luna Phipps-Costin and Donald Pinckney and Ming-Ho Yee and Yangtian Zi and Carolyn Jane Anderson and Molly Q Feldman and Arjun Guha and Michael Greenberg and Abhinav Jangda},
      year={2022},
      eprint={2208.08227},
      archivePrefix={arXiv},
      primaryClass={cs.LG},
      url={https://arxiv.org/abs/2208.08227}, 
}

@misc{grattafiori2024llama3herdmodels,
      title={The Llama 3 Herd of Models}, 
      author={Aaron Grattafiori and Abhimanyu Dubey and Abhinav Jauhri and Abhinav Pandey and Abhishek Kadian and Ahmad Al-Dahle, et al.},
      year={2024},
      eprint={2407.21783},
      archivePrefix={arXiv},
      primaryClass={cs.AI},
      url={https://arxiv.org/abs/2407.21783}, 
}

@article{raffel2020exploring,
  title={Exploring the Limits of Transfer Learning with a Unified Text-to-Text Transformer},
  author={Raffel, Colin and Shazeer, Noam and Roberts, Adam and Lee, Katherine and Narang, Sharan and Matena, Michael and Zhou, Yanqi and Li, Wei and Liu, Peter J.},
  journal={Journal of Machine Learning Research},
  volume={21},
  number={140},
  pages={1--67},
  year={2020}
}

@article{weber2024redpajama,
  title={RedPajama: an Open Dataset for Training Large Language Models},
  author={Weber, Maurice and Fu, Daniel Y. and Anthony, Quentin and Oren, Yonatan and Adams, Shane and Alexandrov, Anton and Lyu, Xiaozhong and Nguyen, Huu and Yao, Xiaozhe and Adams, Virginia and Athiwaratkun, Ben and Chalamala, Rahul and Chen, Kezhen and Ryabinin, Max and Dao, Tri and Liang, Percy and R{\'e}, Christopher and Rish, Irina and Zhang, Ce},
  journal={Advances in Neural Information Processing Systems},
  volume={37},
  year={2024}
}

@inproceedings{soldaini2024dolma,
  title={Dolma: an Open Corpus of Three Trillion Tokens for Language Model Pretraining Research},
  author={Luca Soldaini and Rodney Kinney and Akshita Bhagia, et al.},
  booktitle={Proceedings of the 62nd Annual Meeting of the Association for Computational Linguistics (Volume 1: Long Papers)},
  pages={15725--15788},
  address={Bangkok, Thailand},
  publisher={Association for Computational Linguistics},
  doi={10.18653/v1/2024.acl-long.840},
  year={2024}
}

@inproceedings{sachdeva2026askllm,
  title={How to train data-efficient {LLM}s},
  author={Noveen Sachdeva and Benjamin Coleman and Wang-Cheng Kang and Jianmo Ni and Lichan Hong and Ed H. Chi and James Caverlee and Julian McAuley and Derek Zhiyuan Cheng},
  booktitle={The Fourteenth International Conference on Learning Representations},
  year={2026},
  url={https://openreview.net/forum?id=yKUbw7q1IA}
}

@InProceedings{wettig2024qurating,
  title = 	 {{Q}u{R}ating: Selecting High-Quality Data for Training Language Models},
  author =       {Wettig, Alexander and Gupta, Aatmik and Malik, Saumya and Chen, Danqi},
  booktitle = 	 {Proceedings of the 41st International Conference on Machine Learning},
  year = 	 {2024}
}

@article{dai2025dataefficacy,
  title={Data Efficacy for Language Model Training},
  author={Dai, Yalun and Huang, Yangyu and Zhang, Xin and Wu, Wenshan and Li, Chong and Lu, Wenhui and Cao, Shijie and Dong, Li and Li, Scarlett},
  journal={arXiv preprint arXiv:2506.21545},
  year={2025}
}

@inproceedings{yu2025groupmates,
  title={Group-Level Data Selection for Efficient Pretraining},
  author={Yu, Zichun and Peng, Fei and Lei, Jie and Overwijk, Arnold and Yih, Wen-tau and Xiong, Chenyan},
  booktitle={Advances in Neural Information Processing Systems},
  year={2025},
  url={https://openreview.net/forum?id=uX4dyc7Z5Z}
}

@inproceedings{wang2025datashapley,
  title={Data Shapley in One Training Run},
  author={Wang, Jiachen T. and Mittal, Prateek and Song, Dawn and Jia, Ruoxi},
  booktitle={The Thirteenth International Conference on Learning Representations},
  year={2025},
  url={https://openreview.net/forum?id=HD6bWcj87Y}
}

@inproceedings{su2024nemotroncc,
  title={Nemotron-cc: Transforming common crawl into a refined long-horizon pretraining dataset},
  author={Su, Dan and Kong, Kezhi and Lin, Ying and Jennings, Joseph and Norick, Brandon and Kliegl, Markus and Patwary, Mostofa and Shoeybi, Mohammad and Catanzaro, Bryan},
  booktitle={Proceedings of the 63rd Annual Meeting of the Association for Computational Linguistics (Volume 1: Long Papers)},
  pages={2459--2475},
  year={2025}
}

@inproceedings{bukharin2024diversity,
  title={Data Diversity Matters for Robust Instruction Tuning},
  author={Bukharin, Alexander and Li, Shiyang and Wang, Zhengyang and Yang, Jingfeng and Yin, Bing and Li, Xian and Zhang, Chao and Zhao, Tuo and Jiang, Haoming},
  booktitle={Findings of the Association for Computational Linguistics: EMNLP 2024},
  pages={3411--3425},
  address={Miami, Florida, USA},
  publisher={Association for Computational Linguistics},
  doi={10.18653/v1/2024.findings-emnlp.195},
  url={https://aclanthology.org/2024.findings-emnlp.195/},
  year={2024}
}

@inproceedings{chen2025mig,
  title={MIG: Automatic Data Selection for Instruction Tuning by Maximizing Information Gain in Semantic Space},
  author={Chen, Yicheng and Li, Yining and Hu, Kai and Ma, Zerun and HaochenYe, HaochenYe and Chen, Kai},
  booktitle={Findings of the Association for Computational Linguistics: ACL 2025},
  pages={9902--9915},
  address={Vienna, Austria},
  publisher={Association for Computational Linguistics},
  doi={10.18653/v1/2025.findings-acl.515},
  url={https://aclanthology.org/2025.findings-acl.515/},
  year={2025}
}

@inproceedings{ge2024clustering,
  title={Clustering and Ranking: Diversity-preserved Instruction Selection through Expert-aligned Quality Estimation},
  author={Ge, Yuan and Liu, Yilun and Hu, Chi and Meng, Weibin and Tao, Shimin and Zhao, Xiaofeng and Xia, Mahong and Li, Zhang and Chen, Boxing and Yang, Hao and Li, Bei and Xiao, Tong and Zhu, JingBo},
  booktitle={Proceedings of the 2024 Conference on Empirical Methods in Natural Language Processing},
  pages={464--478},
  address={Miami, Florida, USA},
  publisher={Association for Computational Linguistics},
  doi={10.18653/v1/2024.emnlp-main.28},
  url={https://aclanthology.org/2024.emnlp-main.28/},
  year={2024}
}

@inproceedings{zhang2025d3, 
  author = {Zhang, Jia and Zhang, Chen-Xi and Liu, Yao and Jin, Yi-Xuan and Yang, Xiao-Wen and Zheng, Bo and Liu, Yi and Guo, Lan-Zhe}, 
  title = {D3: diversity, difficulty, and dependability-aware data selection for sample-efficient llm instruction tuning}, 
  year = {2025}, 
  isbn = {978-1-956792-06-5}, 
  booktitle = {Proceedings of the Thirty-Fourth International Joint Conference on Artificial Intelligence}, 
  articleno = {928}, 
  numpages = {9}, 
  location = {Montreal, Canada}, 
  series = {IJCAI '25} 
}

@misc{pandia2026coverage,
  title={Rethinking Data Selection: The Importance of Coverage over Difficulty in Generative Fine-Tuning},
  author={Lalchand Pandia and Kanishka Misra and Allyson Ettinger},
  year={2026},
  url={https://openreview.net/forum?id=g1DiK2Yi4j}
}

@inproceedings{chen2026dqo,
  title={Post-training large language models for diverse high-quality responses},
  author={Chen, Yilei and Chakraborty, Souradip and Wolf, Lorenz and Paschalidis, Ioannis and Pacchiano, Aldo},
  booktitle={International Conference on Learning Representations},
  volume={2026},
  pages={127946--127975},
  year={2026}
}

@inproceedings{zeng2024diversified,
  title={On Diversified Preferences of Large Language Model Alignment},
  author={Zeng, Dun and Dai, Yong and Cheng, Pengyu and Wang, Longyue and Hu, Tianhao and Chen, Wanshun and Du, Nan and Xu, Zenglin},
  booktitle={Findings of the Association for Computational Linguistics: EMNLP 2024},
  pages={9194--9210},
  address={Miami, Florida, USA},
  publisher={Association for Computational Linguistics},
  doi={10.18653/v1/2024.findings-emnlp.538},
  url={https://aclanthology.org/2024.findings-emnlp.538/},
  year={2024}
}

@article{zhang2025beyondsingle,
  title={Beyond Single: A Data Selection Principle for LLM Alignment via Fine-Grained Preference Signals},
  author={Zhang, Jia and Liu, Yao and Zhang, Chen-Xi and Liu, Yi and Jin, Yi-Xuan and Guo, Lan-Zhe and Li, Yu-Feng},
  journal={arXiv preprint arXiv:2508.07638},
  year={2025},
  doi={10.48550/arXiv.2508.07638}
}

@inproceedings{shen2025seal,
  title={SEAL: Safety-enhanced Aligned LLM Fine-tuning via Bilevel Data Selection},
  author={Shen, Han and Chen, Pin-Yu and Das, Payel and Chen, Tianyi},
  booktitle={The Thirteenth International Conference on Learning Representations},
  year={2025},
  url={https://openreview.net/forum?id=VHguhvcoM5}
}

@inproceedings{jha2025rlguided,
  title={{RL}-Guided Data Selection for Language Model Finetuning},
  author={Animesh Jha and Ananjan Nandi and Harshit Gupta},
  booktitle={NeurIPS 2025 Workshop: Reliable ML from Unreliable Data},
  year={2025},
  url={https://openreview.net/forum?id=YfMIkHYxP0}
}

@inproceedings{wang2025reschedule,
  title={Scheduling your llm reinforcement learning with reasoning trees},
  author={Wang, Hong and Hao, Zhezheng and Luo, Jian and Wei, Chenxing and Shu, Yao and Liu, Lei and Lin, Qiang and Dong, Hande and Chen, Jiawei},
  booktitle={International Conference on Learning Representations},
  volume={2026},
  pages={154734--154753},
  year={2026}
}

@inproceedings{ming2026ideal,
  title={IDEAL: Data Equilibrium Adaptation for Multi-Capability Language Model Alignment},
  author={Ming, Chenlin and Qu, Chendi and Cai, Mengzhang and Pei, Qizhi and Pan, Zhuoshi and Li, Yu and Duan, Xiaoming and Wu, Lijun and He, Conghui},
  booktitle={The Fourteenth International Conference on Learning Representations},
  year={2026},
  url={https://openreview.net/forum?id=n9wS0Hdvri}
}

@misc{yang2025diversitydrivendataselectionlanguage,
  title={Diversity-driven Data Selection for Language Model Tuning through Sparse Autoencoder}, 
  author={Xianjun Yang and Shaoliang Nie and Lijuan Liu and Suchin Gururangan and Ujjwal Karn and Rui Hou and Madian Khabsa and Yuning Mao},
  year={2025},
  eprint={2502.14050},
  archivePrefix={arXiv},
  primaryClass={cs.CL},
  url={https://arxiv.org/abs/2502.14050}, 
}

@inproceedings{wang2026opusefficientprincipleddata,
  title={{OPUS}: Towards Efficient and Principled Data Selection in Large Language Model Pre-training in Every Iteration},
  author={Shaobo Wang and Xuan Ouyang and Tianyi Xu and Yuzheng Hu and Jialin Liu and Guo Chen and Tianyu Zhang and Junhao Zheng and Kexin Yang and Xingzhang Ren and Dayiheng Liu and Linfeng Zhang},
  booktitle={Forty-third International Conference on Machine Learning},
  year={2026}
}

@misc{abbas2023semdedupdataefficientlearningwebscale,
  title={SemDeDup: Data-efficient learning at web-scale through semantic deduplication}, 
  author={Amro Abbas and Kushal Tirumala and Dániel Simig and Surya Ganguli and Ari S. Morcos},
  year={2023},
  eprint={2303.09540},
  archivePrefix={arXiv},
  primaryClass={cs.LG},
  url={https://arxiv.org/abs/2303.09540}, 
}

@inproceedings{shao2024balanceddatasamplinglanguage,
  title={Balanced data sampling for language model training with clustering},
  author={Shao, Yunfan and Li, Linyang and Fei, Zhaoye and Yan, Hang and Lin, Dahua and Qiu, Xipeng},
  booktitle={Findings of the Association for Computational Linguistics: ACL 2024},
  pages={14012--14023},
  year={2024}
}

@inproceedings{zhang2025harnessing,
  title={Harnessing Diversity for Important Data Selection in Pretraining Large Language Models},
  author={Chi Zhang and Huaping Zhong and Kuan Zhang and Chengliang Chai and Rui Wang and Xinlin Zhuang and Tianyi Bai and Qiu Jiantao and Lei Cao and Ju Fan and Ye Yuan and Guoren Wang and Conghui He},
  booktitle={The Thirteenth International Conference on Learning Representations},
  year={2025},
  url={https://openreview.net/forum?id=bMC1t7eLRc}
}

@ARTICLE{7439920,
  author={Paul, Saurabh and Drineas, Petros},
  journal={Neural Computation}, 
  title={Feature Selection for Ridge Regression with Provable Guarantees}, 
  year={2016},
  volume={28},
  number={4},
  pages={716-742},
  doi={10.1162/NECO_a_00816}
}

@inproceedings{10.5555/3327144.3327172,
  author = {McCurdy, Shannon},
  booktitle = {Advances in Neural Information Processing Systems},
  title = {Ridge Regression and Provable Deterministic Ridge Leverage Score Sampling},
  year = {2018}
}

@INPROCEEDINGS{9747039,
  author={Cherfaoui, Farah and Kadri, Hachem and Ralaivola, Liva},
  booktitle={ICASSP 2022 - 2022 IEEE International Conference on Acoustics, Speech and Signal Processing (ICASSP)}, 
  title={Scalable Ridge Leverage Score Sampling for the Nyström Method}, 
  year={2022},
  volume={},
  number={},
  pages={4163-4167},
  doi={10.1109/ICASSP43922.2022.9747039}
}

@inproceedings{10.5555/3294996.3295140,
  author = {Musco, Cameron and Musco, Christopher},
  title = {Recursive sampling for the Nystr\"{o}m method},
  booktitle = {Advances in Neural Information Processing Systems},
  year = {2017}
}

@article{10.5555/2969239.2969326,
  title={Fast randomized kernel ridge regression with statistical guarantees},
  author={Alaoui, Ahmed and Mahoney, Michael W},
  journal={Advances in Neural Information Processing Systems},
  volume={28},
  year={2015}
}

@book{doi:10.1137/1.9781421407944,
  author = {Golub, Gene H. and Van Loan, Charles F.},
  title = {Matrix Computations - 4th Edition},
  publisher = {Johns Hopkins University Press},
  year = {2013},
  doi = {10.1137/1.9781421407944},
  address = {Philadelphia, PA},
  edition   = {},
  URL = {https://epubs.siam.org/doi/abs/10.1137/1.9781421407944},
  eprint = {https://epubs.siam.org/doi/pdf/10.1137/1.9781421407944}
}

@misc{deepseekai2026deepseekv4,
      title={DeepSeek-V4: Towards Highly Efficient Million-Token Context Intelligence},
      author={DeepSeek-AI},
      year={2026},
}

@misc{kimiteam2026kimik25visualagentic,
      title={Kimi K2.5: Visual Agentic Intelligence}, 
      author={Kimi Team},
      year={2026},
      eprint={2602.02276},
      archivePrefix={arXiv},
      primaryClass={cs.CL},
      url={https://arxiv.org/abs/2602.02276}, 
}

@misc{wang2025ultrafinewebefficientdatafiltering,
      title={Ultra-FineWeb: Efficient Data Filtering and Verification for High-Quality LLM Training Data}, 
      author={Yudong Wang and Zixuan Fu and Jie Cai and Peijun Tang and Hongya Lyu and Yewei Fang and Zhi Zheng and Jie Zhou and Guoyang Zeng and Chaojun Xiao and Xu Han and Zhiyuan Liu},
      year={2025},
      eprint={2505.05427},
      archivePrefix={arXiv},
      primaryClass={cs.CL},
      url={https://arxiv.org/abs/2505.05427}, 
}

@misc{joulin2016bagtricksefficienttext,
      title={Bag of Tricks for Efficient Text Classification}, 
      author={Armand Joulin and Edouard Grave and Piotr Bojanowski and Tomas Mikolov},
      year={2016},
      eprint={1607.01759},
      archivePrefix={arXiv},
      primaryClass={cs.CL},
      url={https://arxiv.org/abs/1607.01759}, 
}

\clearpage
\begingroup
\setcounter{tocdepth}{3}
\linespread{0.5}\selectfont
\tableofcontents
\endgroup


\appendix

\section{Related work}
\label{subsec:recent_progress}

\paragraph{Data selection for pretraining.} Research on LLM pretraining data selection largely evolved from Common-Crawl cleaning and filtering. Early pipelines such as C4 \citep{raffel2020exploring}, RefinedWeb \citep{penedo2023refinedweb}, RedPajama \citep{weber2024redpajama}, SemDeDup \citep{abbas2023semdedupdataefficientlearningwebscale}, FineWeb \citep{penedo2024fineweb}, and Dolma-V1 \citep{soldaini2024dolma} established the standard recipe of heuristic filtering and deduplication, defining the candidate corpora on which later selection methods operate. A more explicit formulation appeared in DSIR \citep{xie2023data}, which selects data by matching a target reference distribution (e.g., Wikipedia and books) via importance resampling in an $n$-gram feature space. Subsequent work moved toward model-based quality filtering, where a lightweight model is trained and scales selection to web-scale corpora, as in DCLM-fastText \citep{li2024datacomp}, FineWeb-Edu \citep{penedo2024fineweb}, QuRater \citep{wettig2024qurating}, AskLLM \citep{sachdeva2026askllm} and DataMan \citep{peng2025dataman}. More recent methods make selection increasingly training-aware: PDS \citep{gu2024data}, LQS \citep{dai2025dataefficacy} and PDPC \citep{zhang2025preference} incorporate model's gradient or perplexity preference during training, while OPUS \citep{wang2026opusefficientprincipleddata}, MATES \citep{yu2024mates}, Group-MATES \citep{yu2025groupmates}, BLISS \citep{hao2025bliss}, and In-Run Data Shapley \citep{wang2025datashapley} use optimizer, model-state, data influence, or bilevel optimization signals to adapt selection to the evolving model. The quality-filtering methods can be further applied within each domain organized by \citep{wettig2025organize} to strengthen the effectiveness. In parallel, another line improves data quality through rewriting rather than filtering alone, including REWIRE \citep{nguyen2025rewire}, RePro \citep{yu2025repro}, and Nemotron-CC \citep{su2024nemotroncc}. Finally, several recent works emphasize that pretraining data selection is fundamentally a set-level optimization problem, not only a pointwise scoring problem: DiSF \citep{fan2025combatting} and Density \citep{sachdeva2026askllm} highlight the role of diversity, DoReMi \citep{xie2023doremi} optimizes source mixture weights, and Quad \citep{zhang2025harnessing}, QuaDMix \citep{liu2025quadmix} and DATAMASK \citep{fan2025joint} jointly optimize multiple signals such as quality and diversity. These developments motivate studying diversified data selection as a corpus-level objective that balances quality, redundancy, and coverage under a fixed training budget.

\paragraph{Data selection beyond pretraining.}
Beyond pretraining, data selection in later LLM training stages has developed along three closely related directions. First, a substantial line of work studies subset selection for instruction tuning, where the main question is how to identify a small but high-value subset from a large instruction pool. Early and concurrent studies emphasize the role of diversity and representative coverage, including sparse autoencoder as representatives \citep{yang2025diversitydrivendataselectionlanguage}, DPP-based subset selection \citep{wang2024diversity}, robustness-oriented diversity analysis \citep{bukharin2024diversity}, systematic diversity measurement and NovelSum \citep{yang2025measuring}, semantic-space information gain maximization \citep{chen2025mig}, diversity-preserved clustering-and-ranking pipelines \citep{ge2024clustering}, and joint criteria over diversity, difficulty, and dependability \citep{zhang2025d3}. Recent evidence further suggests that, in generative fine-tuning, coverage can be more important than difficulty-based scoring alone \citep{pandia2026coverage}. Second, in alignment and post-training, data selection becomes more tightly coupled with preference structure, response diversity, and safety objectives. Representative examples include diversity-quality joint optimization in post-training \citep{chen2026dqo}, modeling diversified human preferences in alignment \citep{zeng2024diversified}, selecting preference data using fine-grained preference signals \citep{zhang2025beyondsingle}, and safety-aware bilevel data selection for aligned fine-tuning \citep{shen2025seal}. Third, several recent works move beyond static subset choice and incorporate scheduling or mixture adaptation into the selection problem itself, such as RL-guided subset search for finetuning \citep{jha2025rlguided}, curriculum-style scheduling for RL-based post-training \citep{wang2025reschedule}, and dynamic mixture rebalancing across capabilities in multi-domain SFT \citep{ming2026ideal}. Overall, compared with pretraining, these later-stage methods typically operate on much smaller but cleaner datasets, place stricter requirements on annotation quality and formatting consistency, and treat data selection less as web-scale noise filtering and more as controlled optimization over diversity, preference structure, curriculum, safety, and capability balance.


\section{Supplementary mathematical analyses}
\label{math}

\subsection{Proof of Theorem \ref{geothm}}
\label{proof0}

\begin{proof}
Let $G=\frac{1}{n}X^\top X + \lambda I_d$. We consider the matrix determinant lemma~\citep{doi:10.1137/1.9781421407944}, a basic relationship in linear algebra as follows.
\[
\det(G+u v^\top) = \det(G)\cdot(1+v^\top G^{-1}u).
\]
Substituting $u,v$ both with $z$, we have
\[
\begin{split}
    \det(G+z z^\top)&=\det(G)\cdot(1+z^\top G^{-1} z)\\
    &=\det(G)\cdot(1+\ell^\lambda(z;X)).
\end{split}
\]
Rearranging this yields
\[
\ell^\lambda(z;X)=\frac{\det(G+z z^\top)}{\det(G)}-1.
\]

\end{proof}

\subsection{Proof of Theorem \ref{reprethm}}
\label{proof}

\begin{proof}
Suppose the SVD of $\frac{1}{\sqrt{n}}X = U\Sigma V^\top$, where \(U\) and \(V\) are orthonormal matrices. So we have \(U^\top U=VV^\top=I_d\). Then, the regularized Gram matrix should become
\[
    \begin{split}
        & \frac{1}{n}X^\top X + \lambda I_d \\
        &= V\Sigma^\top U^\top U\Sigma V^\top + \lambda V V^\top \\
        &= V\Sigma^\top \Sigma V^\top + \lambda V I_d V^\top \\
        &=V(\Sigma^\top \Sigma + \lambda I_d) V^\top.
    \end{split}
\]
Hence, the inverse matrix has the form of
\[
    \begin{split}
        & \left(\frac{1}{n}X^\top X + \lambda I_d\right)^{-1} \\
        &= (V(\Sigma^\top \Sigma + \lambda I_d) V^\top)^{-1} \\
        &= (V^\top)^{-1}(\Sigma^\top \Sigma + \lambda I_d)^{-1} V^{-1}\\
        &= V(\Sigma^\top \Sigma + \lambda I_d)^{-1} V^\top.
    \end{split}
\]
Finally, by definition we have
\[
    \begin{split}
       \ell^{\lambda}(z;X) &= z^\top \left(\frac{1}{n}X^\top X + \lambda I_d\right)^{-1}z \\
        &= z^\top V(\Sigma^\top \Sigma + \lambda I_d)^{-1} V^\top z\\
        &=[\tilde{z}_1,...,\tilde{z}_d]\cdot\text{diag}\left\{\frac{1}{\sigma_1^2+\lambda},...,\frac{1}{\sigma_d^2+\lambda}\right\}\cdot[\tilde{z}_1,...,\tilde{z}_d]^\top \\
        &=\sum_{i=1}^{d} \frac{\tilde{z}_i^2}{\sigma_i^2+\lambda}.
    \end{split}
\]

\end{proof}

\subsection{Interpretation of the occurrence frequencies}
\label{intfre}

We first prove that $\sum_{i=1}^{d}\sigma_i^2 = 1$. 

\begin{proof}
    Consider the trace of the Gram matrix, we have
    \[
        \begin{split}
            \mathrm{Tr}(1/n\cdot X^\top X) &= 1/n\cdot\mathrm{Tr}\left(\sum_{i=1}^{n}x_ix_i^\top\right)\\
            &= 1/n\cdot\sum_{i=1}^{n}\mathrm{Tr}(x_ix_i^\top)\\
            &=1/n\cdot\sum_{i=1}^{n}\mathrm{Tr}(x_i^\top x_i)\\
            &=1/n\cdot\sum_{i=1}^{n}x_i^\top x_i\\
            &=1/n\cdot\sum_{i=1}^{n}\|x_i\|_2^2\\
            &=1.
        \end{split}
    \]
    The third equality comes from the fact that the operator \(\mathrm{Tr}\) is invariant to matrix commutation operation. The last equality holds due to the fact that all embedding vectors are normalized.

    On the other hand, considering the SVD of matrix $X$, we have
    \[
        \begin{split}
            \mathrm{Tr}(1/n\cdot X^\top X) &= \mathrm{Tr}(V\Sigma ^\top U^\top U \Sigma V^\top)\\
            &= \mathrm{Tr}(V\Sigma^\top \Sigma V^\top)\\
            &=\mathrm{Tr}(V^\top V\Sigma^\top \Sigma)\\
            &=\mathrm{Tr}(\Sigma^\top \Sigma)\\
            &=\sum_{i=1}^{d} \sigma_i^2.
        \end{split}
    \]

    Combining the above two relations, we have
    \[
        \sum_{i=1}^{d} \sigma_i^2=1.
    \]
\end{proof}

In the following elaborations, we use a simple case to show why \(\{\sigma_i^2\}_{i=1}^{d}\) can be interpreted as frequencies. We already have
\(
\sum_{i=1}^{d}\sigma_i^2 = 1.
\)
Given the non-negativity of $\sigma_i^2 \ge 0$, the set $\{\sigma_i^2\}$ can be viewed as a valid frequency distribution. Next, we consider an extreme and trivial case. Suppose the embedding samples are $X=[x_1, \ldots, x_n]^\top$, but the dataset consists of only the orthonormal bases $\{v_1, \ldots, v_d\}$ of the space $\mathbb{R}^d$. Therefore, some samples in $X$ acutally occur multiple times. Let the frequencies of $\{v_1, \ldots, v_d\}$ be $\{p_1, \ldots, p_d\}$. Then we have
\[
\begin{split}
    \frac{1}{n} X^\top X &= \frac{1}{n}\sum_{i=1}^{n} x_i x_i^\top\\
    &= \sum_{i=1}^{d} p_i \, v_i v_i^\top.
\end{split}
\]
Now consider the singular value decomposition of $X$, there exists some orthonormal matrix $U$ such that 
\[
1/\sqrt{n}\cdot X = U \,\mathrm{diag}\{\sqrt{p_1}, \ldots, \sqrt{p_d}\}\, V^\top =U\Sigma V^\top.
\]
Consequently, we obtain $\sigma_i^2 = p_i$. Extending this intuition to a general dataset $X$, we can intuitively regard the variance $\sigma_i^2$ of the principal component $v_i$ as its occurrence frequency.

\subsection{The choice of the regularization parameter}
\label{morelev}
We show that \(\lambda\) should be a small positive number such that \(\lambda\ll d^{-1}\).

We first consider the case of a perfectly diverse distribution where no more samples contribute to increasing the diversity of the data. In this case, all eigenvalues of the Gram matrix should be equal: $\sigma_1^2=\cdots=\sigma_d^2=1/d$. If we ensure \(\lambda\ll d^{-1}\), then every sample has a leverage score $\ell^\lambda(z) \approx d$, because $$\ell^\lambda(z) = \sum_{i=1}^{d} \frac{\tilde{z}_i^2}{1/d+\lambda}\approx \sum_{i=1}^{d} d{\tilde{z}_i^2}=d\|V^\top z\|_2^2=d.$$ 

Now if the distribution is not perfectly diverse, then there exists a sample $z$ such that $\ell^\lambda(z) > d$. Then there exist some principal semantic vectors such that $\sigma_i^2 < 1/d$, and $z$ is represented mostly along these vectors. Specifically, if a principal semantic vector $v_i$ is rarely covered in the selected data, then its corresponding \(\sigma_i^2\) is approximately zero. Therefore, if we set \(\lambda\) as a very small numeric, then the uncovered semantic signals \(\frac{\tilde{z}_i^2}{\sigma_i^2+\lambda}\approx \frac{\tilde{z}_i^2}{\lambda}\) will be significantly highlighted. Empirically, setting \(\lambda = (100\times d)^{-1}\) is sufficient for the above purposes. 

\section{Implementation details of experiments}
\label{detailsetup}

\subsection{Datasets}

\paragraph{FineWeb~\cite{penedo2024fineweb}.}
FineWeb is a large-scale English web corpus constructed from CommonCrawl and released as a cleaned and deduplicated pretraining dataset. FineWeb contains more than 18.5 trillion tokens of web text and is explicitly designed as a high-quality training corpus for large language models. Its pipeline applies large-scale web data cleaning and deduplication, with the stated goal of improving downstream language modeling performance while preserving the broad coverage and natural diversity of web text. The dataset is built using the \texttt{datatrove} processing framework and is intended as a strong web-data baseline for pretraining research.

For our own web data, we process a CommonCrawl dump with \texttt{datatrove} following similar pipeline steps in FineWeb, including heuristic filtering and deduplication while deliberately excluding quality filtering to preserve the original degree of data diversity. This results in a total amount of around 200B tokens web data.

\paragraph{Datatrove~\cite{penedo2024fineweb}.}
\texttt{datatrove} is a large-scale text data processing library developed for filtering, transforming, and deduplicating raw corpora such as CommonCrawl. The library provides a collection of modular processing blocks together with an extensible framework for adding custom functionality. Its pipeline abstraction is platform-agnostic and can run locally or on cluster environments such as Slurm, while maintaining relatively low memory usage. This design makes \texttt{datatrove} particularly suitable for reproducible construction of web-scale pretraining datasets, and FineWeb is one of its flagship use cases.

\paragraph{StarCoderData~\cite{li2023starcoder}.}
StarCoderData is the training corpus used for StarCoder and StarCoderBase. It is a large-scale collection containing 783GB of code across 86 programming languages, together with 54GB of GitHub Issues, 13GB of Jupyter notebooks in script and text-code formats, and 32GB of GitHub commits, amounting to approximately 250 billion tokens in total. The dataset is designed to support code language modeling at scale while going beyond plain source files by incorporating additional software-development artifacts such as issues and notebooks. As a result, it provides a broad substrate for studying code generation, code completion, and code understanding across multiple programming languages.

For our own code data, we exclusively consider six programming languages including Python, Java, JavaScript, Go, PHP, and Ruby, aiming to obtain clearer signals from the models' downstream performance. The results in a total amound of around 80B tokens code data.

\subsection{Embedding model}

\paragraph{Qwen3-Embedding-0.6B~\cite{zhang2025qwen3}.} For the computation of the embedding vectors of the samples, we adopt Qwen3-Embedding-0.6B. This is a lightweight text embedding model from the Qwen3 Embedding series, built on the dense Qwen3 foundation models and designed for efficient semantic representation and retrieval tasks. Despite its compact 0.6B parameter size, it inherits the strong multilingual understanding and reasoning capabilities of the Qwen3 family, supporting over 100 languages and various programming languages. The model is optimized for a wide range of embedding applications such as text retrieval, code retrieval, clustering, and classification. It supports long-context inputs of up to 32k tokens and allows flexible user-defined embedding dimensions (from 32 up to 4096), enabling adaptation to different storage or latency constraints. For CommonCrawl data, we truncate each sample to the first 512 tokens and compute the corresponding embedding vectors. For StarCoderData corpus, we truncate each sample to the first 1500 tokens for embedding computation. The computed embeddings are then used for all diversity-based methods, including Lev, DiSF and Density.


\subsection{Baselines}

\paragraph{DiSF~\cite{fan2025combatting}.}
DiSF is a diversity-oriented data selection method designed to mitigate dimensional collapse in language model pretraining data selection. It operates by encouraging a more uniform eigenvalue spectrum of the data covariance matrix, thereby promoting richer and more expressive feature representations. We provide a concise and self-contained introduction of this method below.

Let $\mathcal{D} = \{z_1, \dots, z_D\} \subset \mathbb{R}^d$ denote the full corpus of embeddings, and let $\mathcal{S} \subset \mathcal{D}$ denote the selected subset. The goal of DiSF is to construct $\mathcal{S}$ such that the representation space avoids dimensional collapse by encouraging a uniform spectrum of the feature covariance matrix.

Let the normalized feature covariance matrix of a batch $\mathcal{B}$ be defined as
\[
C(\mathcal{B}) \;=\; \frac{1}{|\mathcal{B}|-1} \sum_{z \in \mathcal{B}} \tilde{z}\tilde{z}^\top,
\]
where $\tilde{z}=(z-\mu)/\sigma$ denotes the embedding vector normalized with batch mean $\mu$ and batch standard deviation $\sigma$. Suppose $C$ has eigenvalues $\{\lambda_i\}_{i=1}^d$. Then the following equality holds.
\[
\mathrm{Var}(\{\lambda_i\}_{i=1}^d)
\;=\;
 \sum_{i=1}^d \left(\lambda_i - \frac{1}{d}\sum_{i=1}^{d}\lambda_i\right)^2
\;=\;
 \|C\|_F^2 - d.
\]
Equivalently, up to a constant scaling factor, minimizing the eigenvalue variance is equivalent to minimizing $\|C\|_F^2 - d$. This characterization provides the theoretical foundation for using the Frobenius norm as a proxy for spectral uniformity. Therefore, DiSF formulates data selection as a greedy maximization problem based on a marginal score function. Specifically, at each step, it selects the sample $z_i \in \mathcal{D} \setminus \mathcal{S}$ that maximizes
\[
f(z_i) \;=\; \exp\!\left(-\big\| C(\mathcal{S} \cup \{z_i\}) \big\|_F \right),
\]
which favors samples that reduce the Frobenius norm of the covariance matrix when added to the current set. In practice, the algorithm first separate the corpus into batches of size \(B\) and conduct greedy selection within each batch. Equivalently, the overall objective can be interpreted as approximately minimizing the spectral concentration of the selected set
\[
\min_{\mathcal{S} \subset \mathcal{B}, \, |\mathcal{S}| = \alpha B} \;\; \big\| C(\mathcal{S}) \big\|_F,
\]
where $\alpha \in (0,1)$ is the selection ratio. 

\paragraph{Density~\cite{sachdeva2026askllm}.}
Density sampling is a data selection strategy based on density estimation in the embedding space. The Density method aims to construct a subset $\mathcal{S} \subset \mathcal{D}$ that improves coverage of the corpus by downweighting samples from highly populated regions and upweighting those from relatively sparse regions. This approach ensures that both dense and sparse regions of the data space are adequately represented, leading to improved coverage and density-uniformity in training data. We provide a concise and self-contained introduction of this method here.

Let $\mathcal{D} = \{x_1, \dots, x_D\} \subset \mathbb{R}^d$ denote the full corpus represented in an embedding space. We approximate the sample density through a randomized hashing scheme based on Signed Random Projections (SRP), which provides an efficient sketch of local sample concentration.

Let $R$ denote the number of independent hash functions and $H$ denote the hash range. For each hash function $r \in \{1, \dots, R\}$, the method draws random projection directions and maps each sample $x \in \mathbb{R}^d$ to a binary code according to the signs of its projected coordinates. This code is then converted into an integer bucket index
\[
h_r(x) \in \{0, 1, \dots, H-1\}.
\]
Each hash function therefore partitions the embedding space into random cells, and samples that fall into the same bucket are treated as locally similar under that projection.

After hashing the corpus, a histogram sketch is constructed
\[
S_{r,b} = \sum_{j=1}^{D} \mathbf{1}[h_r(x_j) = b],
\]
which counts how many samples fall into bucket $b$ under hash function $r$. The density score of a sample $x_i$ is then estimated as the average bucket occupancy across all hash functions:
\[
\mathrm{score}(x_i)
=
\frac{1}{R} \sum_{r=1}^{R} S_{r,h_r(x_i)}.
\]
A larger score indicates that $x_i$ lies in a denser region of the embedding space.

Sampling probabilities are defined inversely with respect to this density estimate. Specifically, the unnormalized weight is
\[
w(x_i) \propto \frac{1}{\mathrm{score}(x_i)}.
\]
Thus, samples from dense regions receive smaller weights, while samples from sparse regions receive larger weights. A subset of size $k$ is then drawn without replacement according to these inverse-density weights.

In our implementation, the sketch uses $R = 1024$ hash functions and a hash range of $H = 262{,}144$ buckets. Each hash function is constructed from
\[
m = \lceil \log_2 H \rceil
\]
random hyperplanes, yielding an $m$-bit binary code (with $m=18$ in our setting). The resulting projection tensor therefore has shape $R \times d \times m$.

For efficiency, the corpus is processed in mini-batches of size $20{,}000$ when constructing the histogram sketch $S \in \mathbb{N}^{R \times H}$, and all hashing operations are executed on GPU. Random projections are initialized with a fixed seed to ensure reproducibility, while a separate seed is used during sampling to obtain deterministic results.

\paragraph{DSIR~\cite{xie2023data}.}
DSIR (Data Selection via Importance Resampling) selects training samples by aligning the empirical distribution of the selected dataset with a target domain distribution. It leverages n-gram text features to estimate importance weights, favoring samples that better match domains such as Wikipedia or books. This method is particularly effective for domain adaptation and improving performance on targeted evaluation benchmarks. However, its targeted domain-matching behavior sometimes induces serious dimensional collpase problem.

\paragraph{DCLM-fastText~\cite{li2024datacomp}.}
DCLM-fastText is a quality-based data scoring method introduced in the DataComp-LM benchmark. It trains a lightweight fastText binary classifier to distinguish high-quality text from noisy web data, thereby enabling efficient large-scale filtering. Importantly, the positively labeled reference data is not drawn from conventional curated corpora such as Wikipedia or books. Instead, the classifier is trained on a mixture of \textbf{OpenHermes 2.5} and high-scoring posts from the Reddit \textbf{r/ExplainLikeImFive (ELI5)} community. Concretely, the DataComp-LM pipeline samples examples from OpenHermes 2.5, a synthetic instruction-following dataset for general-purpose assistants, and combines ELI5 question threads with their top-scoring answers selected according to community karma scores. This positive set is therefore dominated by high-quality \emph{instruction-formatted} and \emph{question-answer} style text that spans a broad range of topics, rather than encyclopedic or book-style prose. Negative examples are sampled randomly from a RefinedWeb-style web corpus. After training, the classifier assigns each document a quality score based on its resemblance to this positive distribution, and the corpus is filtered by retaining the top-scoring documents. Owing to the efficiency of fastText, this approach scales well to web-scale corpora while yielding strong empirical performance.


\subsection{Pretrained model configurations and training details}
\label{model_details}

We train all 1.1B models using the Megatron framework on 60B pre-tokenized tokens. With a batch size of 1024 and a context length of 2048, this configuration results in 28,610 training iterations. Detailed model hyperparameters are provided in Table~\ref{tab:llama31_1p1b_pretrain}. The training and evaluation pipeline proceeds as follows: First, the selected corpus is converted into a pre-tokenized binary format compatible with Megatron. Pre-training is then performed under a distributed setup. Upon completion, the Megatron-format checkpoints are converted to the Hugging Face format to facilitate standardized evaluation. Finally, we assess the models using OLMES, which provides a unified pipeline for measuring performance across benchmarks.


\begin{table}[ht]
\centering
\renewcommand{\arraystretch}{1.2}
\caption{Detailed configurations of the 1.1B model of Llama 3.1 architecture.}
\vspace{10pt}
\label{tab:llama31_1p1b_pretrain}
\begin{tabular}{ll}
\toprule
Parameter name & Value \\
\midrule
Parameter number & 1,143,556,096 \\
Hidden size & 2048 \\
Intermediate Hidden Size & 4096 \\
Context Len & 2048 \\
Heads & 32 \\
Layers & 20 \\
Minimum learning rate & 4e-5 \\
Maximum learning rate & 4e-4 \\
$\beta_1$ of optimizer & 0.9 \\
$\beta_2$ of optimizer & 0.95 \\
Warmup steps & 1000 \\
Batch size & 1024$\times$2048 tokens \\
Weight decay & 0.1 \\
Activation function & SwiGLU \\
Gradient clipping threshold & 1.0 \\
\bottomrule
\end{tabular}
\end{table}

\subsection{Benchmarks}

\paragraph{ARC-Easy~\citep{clark2018think}.}
ARC-Easy is the easier subset of the AI2 Reasoning Challenge, a benchmark of grade-school science multiple-choice questions. Each example consists of a natural-language science question and a small set of answer candidates, and the task evaluates factual understanding together with elementary scientific and commonsense reasoning. According to the Hugging Face dataset card, ARC-Easy contains 2,251 training examples, 570 validation examples, and 2,376 test examples. It is commonly used to assess broad knowledge recall and relatively direct reasoning ability under a multiple-choice setup. 

\paragraph{ARC-Challenge~\citep{clark2018think}.}
ARC-Challenge is the harder subset of the AI2 Reasoning Challenge. In contrast to ARC-Easy, its questions were selected to be difficult for simple retrieval-based and co-occurrence-based baselines, making the benchmark substantially more reasoning-intensive. The Hugging Face dataset card reports 1,119 training examples, 299 validation examples, and 1,172 test examples. The benchmark is designed to probe scientific question answering under more demanding compositional and commonsense reasoning conditions.

\paragraph{WinoGrande~\citep{sakaguchi2020winogrande}.}
WinoGrande is a large-scale commonsense reasoning benchmark inspired by the Winograd Schema Challenge. It is formulated as a fill-in-the-blank task with two candidate options, where the model must resolve an ambiguous reference using commonsense knowledge rather than shallow lexical cues. The Hugging Face dataset card describes WinoGrande as a collection of roughly 44K problems. In the commonly used debiased configuration on Hugging Face, the dataset contains 9,248 training examples, 1,267 validation examples, and 1,767 test examples. This benchmark primarily evaluates pronoun resolution and referential commonsense reasoning.

\paragraph{HellaSwag~\citep{zellers2019hellaswag}.}
HellaSwag is a benchmark for commonsense natural language inference and grounded sentence completion. Each instance presents a context and several candidate continuations, and the model must choose the most plausible continuation. The dataset is constructed adversarially so that the incorrect endings are highly confusable, making the task difficult for models despite being relatively easy for humans. The Hugging Face dataset card reports a total of 59,950 examples. HellaSwag is widely used to evaluate everyday event understanding, plausibility judgment, and broad commonsense inference.

\paragraph{PIQA~\citep{bisk2020piqa}.}
PIQA, short for Physical Interaction Question Answering, is a benchmark for physical commonsense reasoning. Its questions are grounded in everyday goal-oriented situations, often involving tools, materials, and basic interactions with physical objects. The Hugging Face dataset card notes that PIQA is motivated by practical scenarios such as those found on instructional websites, and reports 16,000 training examples, 2,000 development examples, and 3,000 test examples. The benchmark primarily measures whether a model can reason about plausible physical actions and affordances.

\paragraph{Social IQa~\citep{sap2019socialiqa}.}
Social IQa is a multiple-choice question answering benchmark for social commonsense intelligence. Each example describes a social situation and asks a question about motives, reactions, or likely consequences, requiring the model to reason about human behavior and social implications. The Hugging Face dataset card states that the benchmark contains over 37,000 question-answer pairs. In the Hugging Face release, the labeled splits contain 33,410 training examples and 1,954 validation examples. Social IQa is intended to evaluate models' ability to infer mental states, intentions, and likely social outcomes.

\paragraph{CommonSenseQA~\citep{talmor2019commonsenseqa}.}
CommonSenseQA is a multiple-choice commonsense question answering benchmark that requires diverse forms of everyday background knowledge. Each question is paired with one correct answer and four distractors, making it a five-way classification problem. The Hugging Face dataset card reports that the dataset contains 12,102 questions in total. The benchmark is designed to test broad commonsense reasoning, including taxonomic, causal, spatial, and functional knowledge.

\paragraph{Multilingual MBPP~\citep{cassano2022multiplescalableextensibleapproach}.}
The multilingual MBPP benchmark used in our experiments corresponds to the MBPP portion of MultiPL-E. MultiPL-E extends code generation benchmarks to multiple programming languages by translating the original Python tasks into target languages while preserving executable test-based evaluation. According to the Hugging Face dataset card, MultiPL-E supports 22 programming languages. The underlying MBPP benchmark consists of 974 crowd-sourced programming problems designed for entry-level programmers, covering programming fundamentals and standard library usage. In our setting, we evaluate the six language subsets corresponding to the programming languages considered in the main experiments. This benchmark primarily measures multilingual code synthesis from natural-language problem descriptions.

\paragraph{HumanEval~\citep{chen2021evaluating}.}
HumanEval is a code generation benchmark consisting of handwritten Python programming problems with unit-test-based evaluation. Each example contains a function signature, a natural-language docstring, a canonical solution, and a hidden correctness test. The Hugging Face dataset card reports that HumanEval contains 164 problems and only provides a test split. The benchmark is widely used to evaluate functional code generation accuracy under an execution-based metric, and is particularly useful for assessing whether a model can synthesize correct code from concise natural-language specifications.

\newpage

\subsection{Diversity metrics}
\label{div_define}

\noindent
Let $X = \{x_1, \dots, x_n\} \subset \mathbb{R}^d$ denote the embedding vectors of $n$ samples where each vector has norm $\|x_i\|_2=1$. Let $d(x_i, x_j)$ denote the $L_2$ distance between $x_i$ and $x_j$. We denote by $N_k(x_i)$ the $k$-th nearest neighbor of $x_i$ under $d(\cdot,\cdot)$. We explicitly define the seven diversity metrics below.

\paragraph{CluIne~\cite{du2019boosting}.}
Let $\{C_k\}_{k=1}^K$ be a partition of $X$ obtained by the clustering algorithm, $K$-means. The cluster inertia is defined as
\[
\mathrm{CluIne}(X) \;=\; \sum_{k=1}^{K} \sum_{x_i \in C_k} \|x_i - \mu_k\|_2^2,
\quad \mu_k = \frac{1}{|C_k|}\sum_{x_j \in C_k} x_j.
\]
This metric quantifies the within-cluster variance of the dataset under the induced partition. A larger value indicates that samples are more dispersed relative to their assigned centroids, reflecting greater global spread across clusters. However, its interpretation depends on the choice of $K$ and the clustering procedure, and it does not explicitly capture fine-grained pairwise relationships or local redundancy. For all metrics that incorporates $K$-means, we set $K=25$ for the training of clustering.

\paragraph{LogDet~\cite{wang2024diversity}.}
The log-determinant diversity is defined as
\[
\mathrm{LogDet}(X) \;=\; \frac{1}{n}\log \det \big( L ),
\]
where $L$ is the element-wise exponential of the similarity matrix of $X$ such that $L_{ij}=\exp(2x_i^\top x_j)$. This metric measures the volume spanned by the sample embeddings in the induced feature space. A larger value indicates that the samples are more linearly independent and less redundant, thereby capturing global diversity through the geometry of the representation space.

\paragraph{CosDist~\cite{yang2025measuring}.}
The cosine-distance-based diversity is defined as
\[
\mathrm{CosDist}(X) \;=\; \frac{1}{n(n-1)}\sum_{i \neq j} \left(1 - x_i^\top x_j\right).
\]
This metric aggregates pairwise dissimilarities between all samples based on cosine distance. It promotes global separation among samples by assigning larger values to configurations where embeddings are mutually dissimilar. Nevertheless, it treats all pairwise relationships uniformly and may overemphasize contributions from distant pairs while underrepresenting local geometric structure.

\paragraph{Entropy.}
Let $\lambda_1, \dots, \lambda_n$ denote the eigenvalues of the Gram matrix $G=\frac{1}{n}X^\top X$. The spectral entropy is defined as
\[
\mathrm{Entropy}(X) \;=\; - \sum_{i=1}^n {\lambda}_i \log {\lambda}_i.
\]
This metric captures how evenly the variance of the dataset is distributed across principal components. A higher entropy indicates that information is spread across multiple directions in the embedding space, implying a more isotropic and diverse representation. In contrast, low entropy corresponds to concentration along a small number of dominant components, which suggests redundancy.

\paragraph{KNN~\cite{stasaski2022semantic}.}
The KNN-based diversity is defined as
\[
\mathrm{KNN}(X) \;=\; \frac{1}{n} \sum_{i=1}^{n} d\big(x_i, N_k(x_i)\big),
\]
typically with $k=1$. This metric focuses on local geometry by measuring the distance between each sample and its nearest neighbors. Larger values indicate that samples are locally isolated, which reflects stronger sample-level uniqueness. Compared to global aggregation methods, this formulation is more sensitive to redundancy among nearby samples.

\paragraph{NovelSum~\cite{yang2025measuring}.}
The NovelSum metric defines dataset diversity as the sum of sample-level novelty scores
\[
\mathrm{NovelSum}(X) \;=\; \sum_{i=1}^{n} v(x_i),
\quad
v(x_i) \;=\; \sum_{j \neq i} w_{ij}^{\alpha} \, \Delta(x_i, x_j),
\]
where $w_{ij} = \phi(\pi_i(j))$ is a monotonically decreasing function of the rank $\pi_i(j)$ of $x_j$ in the sorted list of distances from $x_i$. For instance, $\pi_i(j)=1$ indicates that $x_j$ is the nearest neighbor of $x_i$. Typically, we set \(\phi(\pi_i(j))=\frac{1}{\pi_i(j)}\). The distance is further defined as
\[
\Delta(x_i, x_j) \;=\; \sigma(x_j)^{\beta} \, d(x_i, x_j),
\]
where $d(\cdot,\cdot)$ is the $L_2$ distance metric and $\sigma(x_j)$ is a density factor that approximates the inverse local density around $x_j$. For example, we use $$\sigma(x_j)=\frac{1}{\sum_{k=1}^{K}d(x_j,N_k(x_j))}.$$ This formulation integrates proximity-aware weighting and density-aware scaling, thereby emphasizing differences with nearby samples while accounting for heterogeneous information density in the space. As a result, NovelSum provides a principled estimate of each sample’s unique contribution to the overall dataset.

\paragraph{Vendi score~\cite{pasarkar2023cousins}.}
Let $\tilde{\lambda}_1, \dots, \tilde{\lambda}_n$ denote the normalized eigenvalues of the sample similarity matrix $XX^\top$ (not the Gram matrix, thus different from Entropy). The Vendi Score of order 1 has an equivalent definition of
\[
\mathrm{Vendi}(X) \;=\; \exp\!\left(- \sum_{i=1}^{n} \tilde{\lambda}_i \log \tilde{\lambda}_i \right).
\]
This metric can be interpreted as the effective number of components contributing to the representation. It transforms spectral entropy into a quantity with a more direct interpretation, where larger values correspond to a richer and more uniformly distributed structure in the embedding space.

\paragraph{Discussion.}
These metrics can be broadly categorized into two classes. Distance-based metrics, including CosDist, KNN, and CluIne, primarily quantify diversity through pairwise or local separations in the embedding space, emphasizing sample-level uniqueness or dispersion. Spectral and distributional metrics, including Entropy, Vendi Score, and LogDet, instead characterize diversity through the global structure of the dataset as reflected in the eigen-spectrum of the Gram matrix. NovelSum integrates both perspectives by jointly modeling local proximity, inter-sample differences, and heterogeneous density, thereby providing a more comprehensive characterization of dataset diversity.

\newpage

\section{Detailed experimental results}
\label{appendix:code_results}

\subsection{Diversity metrics}
\label{table_div_detailed}
In Figure~\ref{radar}, we normalize the results of each diversity metric by applying a linear mapping from \([\text{min},\text{max}]\) to \([0.30,1.00]\) to compare the relative values of different methods. The original values are reported in Table~\ref{div_metric}. Lev consistently achieves the best performance.

\begin{table}[ht]
\renewcommand{\arraystretch}{1.25}
\caption{Dataset diversity comparison at 40\% selection ratio. Higher is better.}
\label{div_metric}
\centering
\small
\begin{tabular}{l|ccccccc}
\noalign{\vskip 0.5pt}
\noalign{\hrule height 0.75pt}
\noalign{\vskip 0.5pt} 
Method & CluIne & LogDet & CosDist & Entropy & KNN & NovelSum & Vendi \\ 
\hline 
Lev (ours) & \textbf{0.767} & \textbf{1.638} & \textbf{0.853} & \textbf{5.317} & \textbf{1.076} & \textbf{0.640} & \textbf{203.805} \\ 
DiSF & \textit{0.757} & \textit{1.623} & 0.832 & \textit{5.229} & \textit{1.071} & 0.626 & \textit{186.629} \\
Density & 0.744 & 1.612 & 0.833 & 5.181 & 1.056 & \textit{0.627} & 177.802 \\
\hline
DCLM-fastText & 0.747 & 1.611 & \textit{0.835} & 5.187 & 1.055 & 0.626 & 178.855 \\
DSIR & 0.732 & 1.599 & 0.821 & 5.122 & 1.053 & 0.611 & 167.730 \\
Random & 0.732 & 1.601 & 0.821 & 5.124 & 1.052 & 0.619 & 168.060 \\
\noalign{\vskip 0.5pt}
\noalign{\hrule height 0.75pt}
\noalign{\vskip 0.5pt}
\end{tabular}
\end{table}


\subsection{Downstream performance on CommonCrawl subset selection}
\label{table_CC_detailed}

We provide full results of the downstream performance over seven language understanding tasks across all baseline methods and Lev for the second evaluation in section~\ref{efflev}. The evaluated datasets are selected from CommonCrawl corpus using various diversification strategies at 40\% and 50\% selection ratios.

\begin{table}[ht]
\renewcommand{\arraystretch}{1.25}
\caption{Downstream performance at 40\% and 50\% web data selection.\\ Accuracies (\%) are Reported. Higher is Better.}
\label{cc_perf_full}
\begin{center}
\begin{small}
\begin{tabular}{l|ccccccc|c}
\noalign{\vskip 0.5pt}
\noalign{\hrule height 0.75pt}
\noalign{\vskip 0.5pt}
\multirow{2}{*}{Method} & \multicolumn{8}{c}{40\% web data selection}\\
\cline{2-9}

& ARC-e & ARC-c & WG & HS & PIQA & SIQA & CSQA & Mean \textuparrow \\
\hline
Random & 48.50$ _{(5)}$ & 27.82$ _{(4)}$ & 52.09$ _{(3)}$ & 46.50$ _{(3)}$ & \textbf{69.40}$ _{(1)}$ & \textbf{46.30}$ _{(1)}$ & 43.00$ _{(3)}$ & 47.66$ _{(4)}$ \\
DSIR & 49.10$ _{(4)}$ & {29.69}$ _{(2)}$ & 51.22$ _{(4)}$ & 39.90$ _{(5)}$ & 64.40$ _{(5)}$ & 43.70$ _{(5)}$ & 35.87$ _{(5)}$ & 44.84$ _{(5)}$ \\
Density & 51.50$ _{(3)}$ & 28.50$ _{(3)}$ & \textbf{53.12}$ _{(1)}$ & {47.40}$ _{(2)}$ & 68.50$ _{(3)}$ & 44.40$ _{(4)}$ & 42.42$ _{(4)}$ & 47.98$ _{(3)}$ \\
DiSF & {53.70}$ _{(2)}$ & \textbf{30.12}$ _{(1)}$ & 51.07$ _{(5)}$ & 45.20$ _{(4)}$ & 67.80$ _{(4)}$ & {45.60}$ _{(2)}$ & {44.96}$ _{(2)}$ & {48.35}$ _{(2)}$ \\
Lev (ours) & \textbf{54.20}$ _{(1)}$ & 27.47$ _{(5)}$ & {52.72}$ _{(2)}$ & \textbf{48.60}$ _{(1)}$ & {69.30}$ _{(2)}$ & 45.00$ _{(3)}$ & \textbf{45.05}$ _{(1)}$ & \textbf{48.91}$ _{(1)}$ \\

\noalign{\vskip 0.5pt}
\noalign{\hrule height 0.75pt}
\noalign{\vskip 0.5pt}
\multirow{2}{*}{Method} & \multicolumn{8}{c}{50\% web data selection}\\
\cline{2-9}
& ARC-e & ARC-c & WG & HS & PIQA & SIQA & CSQA & Mean \textuparrow \\
\hline
Random & 48.50$ _{(5)}$ & 27.82$ _{(4)}$ & 52.09$ _{(4)}$ & 46.50$ _{(3)}$ & \textbf{69.40}$ _{(1)}$ & \textbf{46.30}$ _{(1)}$ & \textbf{43.00}$ _{(1)}$ & 47.66$ _{(3)}$ \\
DSIR & \textbf{53.50}$ _{(1)}$ & 27.73$ _{(5)}$ & \textbf{53.04}$ _{(1)}$ & 42.00$ _{(5)}$ & 65.20$ _{(5)}$ & 44.30$ _{(5)}$ & 42.26$ _{(5)}$ & 46.86$ _{(5)}$ \\
Density & 51.70$ _{(4)}$ & 27.99$ _{(3)}$ & 50.75$ _{(5)}$ & \textbf{48.20}$ _{(1)}$ & 66.40$ _{(4)}$ & 45.00$ _{(3)}$ & 42.51$ _{(3)}$ & 47.51$ _{(4)}$ \\
DiSF & {52.90}$ _{(2)}$ & {28.16}$ _{(2)}$ & {52.49}$ _{(2)}$ & 46.40$ _{(4)}$ & 67.80$ _{(3)}$ & 44.50$ _{(4)}$ & {42.59}$ _{(2)}$ & {47.83}$ _{(2)}$ \\
Lev (ours) & 52.70$ _{(3)}$ & \textbf{29.35}$ _{(1)}$ & 52.33$ _{(3)}$ & {47.10}$ _{(2)}$ & {68.30}$ _{(2)}$ & {45.10}$ _{(2)}$ & 42.34$ _{(4)}$ & \textbf{48.17}$ _{(1)}$ \\
\noalign{\vskip 0.5pt}
\noalign{\hrule height 0.75pt}
\noalign{\vskip 0.5pt}
\end{tabular}
\end{small}
\end{center}
\end{table}

\newpage

\subsection{Downstream performance on StarCoderData subset selection}
\label{table_code_detailed}

We provide full results of the code bits-per-byte over seven code-related tasks across all the diversity-based selection methods. The evaluated dataset are selected from StarCoderData corpus using various diversification strategies at 40\% and 50\% selection ratios.

\begin{table}[ht]
\renewcommand{\arraystretch}{1.25}
\caption{Downstream performance at 50\% and 40\% code data selection.\\ Bits-per-byte($\times 10^{-2}$) is reported. Lower is better.}
\label{code_50_40_2}
\begin{center}
\begin{small}
\begin{tabular}{l|ccccccc|c}
\noalign{\vskip 0.5pt}
\noalign{\hrule height 0.75pt}
\noalign{\vskip 0.5pt}
\multicolumn{9}{c}{\textbf{DCLM-fastText} for the web data.}\\
\hline
\multirow{2}{*}{Method} & \multicolumn{8}{c}{50\% code data selection} \\
\cline{2-9}
 & Python & Go & PHP & Ruby & Java & JS & HumanEval & Mean \textdownarrow \\
\hline
Random & 92.65$ _{(3)}$ & {59.22}$ _{(2)}$ & \textbf{52.15}$ _{(1)}$ & 98.89$ _{(4)}$ & {38.90}$ _{(2)}$ & 68.34$ _{(4)}$ & 53.66$ _{(3)}$ & 66.26$ _{(4)}$ \\
Density & {92.28}$ _{(2)}$ & 61.26$ _{(3)}$ & 53.71$ _{(3)}$ & {91.98}$ _{(2)}$ & 40.68$ _{(4)}$ & {63.55}$ _{(2)}$ & 57.44$ _{(4)}$ & 65.84$ _{(3)}$ \\
DiSF & 96.19$ _{(4)}$ & 61.58$ _{(4)}$ & 54.90$ _{(4)}$ & \textbf{91.08}$ _{(1)}$ & \textbf{38.48}$ _{(1)}$ & 64.35$ _{(3)}$ & \textbf{53.22}$ _{(1)}$ & {65.68}$ _{(2)}$ \\
Lev (ours) & \textbf{88.23}$ _{(1)}$ & \textbf{57.26}$ _{(1)}$ & {53.34}$ _{(2)}$ & 92.23$ _{(3)}$ & 40.60$ _{(3)}$ & \textbf{60.69}$ _{(1)}$ & {53.27}$ _{(2)}$ & \textbf{63.66}$ _{(1)}$ \\
\hline
\multirow{2}{*}{Method} & \multicolumn{8}{c}{40\% code data selection} \\
\cline{2-9}
 & Python & Go & PHP & Ruby & Java & JS & HumanEval & Mean \textdownarrow \\
\hline
Random & 92.65$ _{(4)}$ & \textbf{59.22}$ _{(1)}$ & 52.15$ _{(4)}$ & 98.89$ _{(4)}$ & {38.90}$ _{(2)}$ & 68.34$ _{(4)}$ & \textbf{53.66}$ _{(1)}$ & 66.26$ _{(4)}$ \\
Density & {91.61}$ _{(2)}$ & {61.67}$ _{(2)}$ & \textbf{50.25}$ _{(1)}$ & {93.57}$ _{(2)}$ & 41.22$ _{(4)}$ & {62.23}$ _{(2)}$ & 56.07$ _{(4)}$ & {65.23}$ _{(2)}$ \\
DiSF & 91.97$ _{(3)}$ & 62.58$ _{(3)}$ & {51.26}$ _{(2)}$ & \textbf{93.11}$ _{(1)}$ & 39.64$ _{(3)}$ & 64.12$ _{(3)}$ & 54.96$ _{(3)}$ & 65.38$ _{(3)}$ \\
Lev (ours) & \textbf{90.35}$ _{(1)}$ & 66.65$ _{(4)}$ & 51.87$ _{(3)}$ & 94.99$ _{(3)}$ & \textbf{38.43}$ _{(1)}$ & \textbf{58.76}$ _{(1)}$ & {54.11}$ _{(2)}$ & \textbf{65.02}$ _{(1)}$ \\

\noalign{\vskip 0.5pt}
\noalign{\hrule height 0.75pt}
\noalign{\vskip 0.5pt}
\multicolumn{9}{c}{\textbf{Leverage Sampling} for the web data.}\\
\hline
\multirow{2}{*}{Method} & \multicolumn{8}{c}{50\% code data selection} \\
\cline{2-9}
 & Python & Go & PHP & Ruby & Java & JS & HumanEval & Mean \textdownarrow \\
\hline
Random & 94.59$ _{(4)}$ & 60.57$ _{(3)}$ & 52.60$ _{(4)}$ & 92.30$ _{(3)}$ & 38.74$ _{(4)}$ & 66.27$ _{(4)}$ & {56.65}$ _{(2)}$ & 65.96$ _{(4)}$ \\
Density & {89.78}$ _{(2)}$ & \textbf{58.89}$ _{(1)}$ & 52.25$ _{(3)}$ & \textbf{87.96}$ _{(1)}$ & \textbf{36.64}$ _{(1)}$ & {59.45}$ _{(2)}$ & 59.10$ _{(4)}$ & {63.44}$ _{(2)}$ \\
DiSF & 92.51$ _{(3)}$ & 62.17$ _{(4)}$ & {51.16}$ _{(2)}$ & 93.40$ _{(4)}$ & 38.60$ _{(3)}$ & 61.90$ _{(3)}$ & 56.92$ _{(3)}$ & 65.24$ _{(3)}$ \\
Lev (ours) & \textbf{89.58}$ _{(1)}$ & {60.11}$ _{(2)}$ & \textbf{50.26}$ _{(1)}$ & {90.69}$ _{(2)}$ & {37.60}$ _{(2)}$ & \textbf{59.41}$ _{(1)}$ & \textbf{51.25}$ _{(1)}$ & \textbf{62.70}$ _{(1)}$ \\
\hline
\multirow{2}{*}{Method} &\multicolumn{8}{c}{40\% code data selection} \\
\cline{2-9}
 & Python & Go & PHP & Ruby & Java & JS & HumanEval & Mean \textdownarrow \\
\hline
Random & {89.14}$ _{(2)}$ & 63.69$ _{(4)}$ & {52.39}$ _{(2)}$ & 96.05$ _{(4)}$ & 39.88$ _{(3)}$ & \textbf{61.87}$ _{(1)}$ & 56.48$ _{(3)}$ & {65.64}$ _{(2)}$ \\
Density & 94.44$ _{(3)}$ & \textbf{61.98}$ _{(1)}$ & 53.84$ _{(3)}$ & 94.44$ _{(3)}$ & 39.93$ _{(4)}$ & {62.18}$ _{(2)}$ & 57.91$ _{(4)}$ & 66.39$ _{(4)}$ \\
DiSF & 94.70$ _{(4)}$ & {62.04}$ _{(2)}$ & 54.22$ _{(4)}$ & \textbf{92.48}$ _{(1)}$ & \textbf{39.09}$ _{(1)}$ & 63.65$ _{(3)}$ & {53.99}$ _{(2)}$ & 65.74$ _{(3)}$ \\
Lev (ours) & \textbf{88.93}$ _{(1)}$ & 63.34$ _{(3)}$ & \textbf{52.36}$ _{(1)}$ & {92.92}$ _{(2)}$ & {39.15}$ _{(2)}$ & 63.96$ _{(4)}$ & \textbf{50.56}$ _{(1)}$ & \textbf{64.46}$ _{(1)}$ \\

\noalign{\vskip 0.5pt}
\noalign{\hrule height 0.75pt}
\noalign{\vskip 0.5pt}
\end{tabular}
\end{small}
\end{center}
\end{table}

\newpage

\subsection{Quality-filtered CC data shows inferior code performance relative to Lev-selected CC data}
\label{infapp}

To isolate the impact of quality-filtered CC on code performance, we conduct a two-stage analysis. First, pretraining solely on CC data reveals that quality-filtered CC yields markedly higher code BPB than Lev-selected CC (Table~\ref{tab:b}), indicating substantial loss of code-related content. We then augment both CC subsets with dedicated code data from StarCoderData and retrain. Despite this supplementation, the quality-filtered mixture continues to underperform in code evaluation (Table~\ref{tab:a}), confirming that the induced code degradation persists. This robustness holds across two additional code datasets selected by Lev in Table~\ref{comp1_full}.

\begin{table}[ht]
\renewcommand{\arraystretch}{1.35}
\centering
\caption{Comparison between DCLM-fastText and leverage score sampling under\\ different code data. BBP($\times 10^{-2}$) for Code Mean and Accuracy(\%) for Base Mean.}
\label{comp1_full}
\begin{center}
\begin{small}
\begin{tabular}{l|c|c|c}
\noalign{\vskip 0.5pt}
\noalign{\hrule height 0.75pt}
\noalign{\vskip 0.5pt}

Code Data & Web Data & Code Mean \textdownarrow & Base Mean \textuparrow \\
\hline

\multirow{4}{*}{Random} & Lev (top 40\%) & \textbf{65.64}$ _{(1)}$ & 49.72$ _{(3)}$ \\
 & Lev (top 50\%) & 65.96$ _{(2)}$ & 49.85$ _{(2)}$ \\
 & Qual (top 40\%) & 66.26$ _{(3)}$ & \textbf{50.46}$ _{(1)}$ \\
 & Random & 66.48$ _{(4)}$ & 48.76$ _{(4)}$ \\
\hline
\multirow{4}{*}{Lev (top 50\%)} & Lev (top 40\%) & \textbf{62.48}$ _{(1)}$ & 49.90$ _{(2)}$ \\
 & Lev (top 50\%) & 62.70$ _{(2)}$ & 49.17$ _{(4)}$ \\
 & Qual (top 40\%) & 63.66$ _{(3)}$ & \textbf{51.53}$ _{(1)}$ \\
 & Random & 64.46$ _{(4)}$ & 49.28$ _{(3)}$ \\
\hline
\multirow{4}{*}{Lev (top 40\%)} & Lev (top 40\%) & 64.46$ _{(2)}$ & 49.96$ _{(3)}$ \\
 & Lev (top 50\%) & \textbf{63.55}$ _{(1)}$ & 50.04$ _{(2)}$ \\
 & Qual (top 40\%) & 65.02$ _{(4)}$ & \textbf{50.10}$ _{(1)}$ \\
 & Random & 64.75$ _{(3)}$ & 49.81$ _{(4)}$ \\

\noalign{\vskip 0.5pt}
\noalign{\hrule height 0.75pt}
\noalign{\vskip 0.5pt}
\end{tabular}
\end{small}
\end{center}
\end{table}

\newpage
\section{Additional discussions}

\subsection{Why does Lev achieve better diversity optimization results than DiSF?}
\label{divopt}

We answer an equivalent question: \textit{Why the leverage-score objective should be preferable to DiSF's Frobenius objective?} This question can be answered theoretically. We show that Lev is more sensitive in detecting diverse samples. We demonstrate this in two steps.

\textbf{Step 1.}
{Recall that in our Theorem \ref{reprethm}, $\ell(z)=\sum \tilde{z_i}^2/\sigma_i^2$}, where $\tilde{z}_i=v_i^\top z$ is the projection of $z$ on the $i$-th principal vector of $G=X^\top X/n$ and $\sigma_i^2$ be the eigenvalues of $G$. We omit $\lambda I_d$ for leverage scores for brevity, since this does not affect our analysis. The criterion $F(z)=-\vert\vert(1-1/n)G+1/n\cdot zz^\top\vert\vert_F^2$ that {DiSF uses has a similar statement}. We briefly deduce as follows.

Let $f(\eta)=-\vert\vert(1-\eta)G+\eta zz^\top\vert\vert_F^2$. The directional derivative of the matrix function $M(G):=-\vert\vert G\vert\vert_F^2$ toward direction $zz^\top$ is exactly $\frac{df}{d\eta}|_{\eta=0}$. This quantity describes how fast the candidate sample $z$ increases the DiSF criterion $M(G)$.

Let $E=zz^\top - G$, then 
\[
f(\eta)=-\vert\vert G+\eta E\vert\vert_F^2=-Tr((G+\eta E)^2)=-Tr(G^2)-2\eta Tr(GE)-o(\eta^2).
\]

Hence, $f^\prime(\eta)=-2Tr(GE)-o(\eta)$. Let $c=2Tr(G^2)$ be the constant, we obtain
\[
\begin{split}
\frac{df}{d\eta}|_{\eta=0}&=-2Tr(GE)=-2Tr(G(zz^\top-G))=-2Tr(Gzz^\top)+2Tr(G^2)\\
&=-2Tr(z^\top Gz)+c=-2z^\top Gz+c.
\end{split}
\]

Thus, the criterion in DiSF is equivalent to 
\[
d(z)=-z^\top Gz,
\]
which we refer to as DiSF score. {Applying similar deductions in our Theorem \ref{reprethm} gives} 
\[
d(z)=-\sum_{i=1}^{d} \sigma_i^2 \tilde{z}_i^2 \quad (*).
\]

\textbf{Step 2.} Now, we show that {DiSF can be less sensitive than Lev for capturing samples of rare semantics} (i.e. diverse samples) in a very broad scenario. We illustrate this through a case study, in which we consider only three principal components of the Gram matrix for brevity.

Suppose for the currently selected samples $X$, we have three principal semantic vectors $v_1,v_2,v_3$, and their corresponding eigenvalues be $(\sigma_1^2, \sigma_2^2, \sigma_3^2)=(3/4, 1/4-10^{-2}, 10^{-2})$. This means that the third semantic is extremely rare. {Considering maximizing the Shannon entropy of the eigenvalues (e.g., Vendi Score), samples that allocate more weights on the third dimension provide more diversity gains} and should be preferred.

We illustrate the difference of the two methods with an example. Suppose two candidate samples S1 and S2, and suppose the $(\tilde{z}_1^2,\tilde{z}_2^2,\tilde{z}_3^2)$ of the two samples are
\[
\text{S1}:\quad (0, 1/2, 1/2)
\]
\[
\text{S2}:\quad (1/6, 0, 5/6)
\]

Because S2 has more weight on the third dimension, we expect {a good algorithm should prefer S2 over S1}.

For DiSF:
\[
d(\text{S1})=-(1/4-10^{-2})\cdot 1/2-10^{-2}\cdot 1/2=-1/8
\]
\[
d(\text{S2})=-(3/4)\cdot 1/6-10^{-2}\cdot 5/6=-1/8-10^{-2}\cdot 5/6
\]

Thus $d(\text{S1})>d(\text{S2})$, DiSF mistakenly select S1 in this case.

For Lev:
\[
\ell(\text{S1})=1/(1/4-10^{-2})\cdot 1/2+1/10^{-2}\cdot 1/2\approx 52
\]
\[
\ell(\text{S2})=1/(3/4)\cdot 1/6+1/10^{-2}\cdot 5/6\approx 83.6
\]

Thus $\ell(\text{S2})>\ell(\text{S1})$, Lev correctly select S2.

{This sensitivity is rooted in that Lev uses $1/\sigma_i^2$ as the component weights for $\tilde{z}_i^2$ (Theorem \ref{reprethm}), which are more sensitive than the linear weights $-\sigma_i^2$ in DiSF (see Eq. $(*)$), especially when $\sigma_i^2$ is around $0$, $1/\sigma_i^2$ goes to infinity.} For DiSF with the linear weight, if a sample (e.g. sample S2) also allocate some semantic weights on the frequently occurred directions (measured by the $-\sigma_1^2\tilde{z}_1^2-\sigma_2^2\tilde{z}_2^2$ value of sample S2), the contribution of the rare semantic ($-\sigma_3^2\tilde{z}_3^2$) might not provide signals significant enough to make the sample selected. In contrast, the leverage weight $1/\sigma_i^2$ that Lev uses is able to significantly amplify the impact of the rare semantic ($\tilde{z}_3^2/\sigma_3^2 \gg 1$) to make it selected.

This example actually covers a very broad scenario where DiSF could miss these desired samples. Because through the selection process, \textit{as the size of the selected subset $X$ gets larger, samples that partly contains rare semantics will also contain some other semantics that have occurred in $X$ with very high probability}. In this case, {Lev is still able to distinguish samples that are actually more diverse, but DiSF could miss them}, which is exactly what the given example shows.

\subsection{Insufficient recall of code data is a general problem among open-source data scorers}
\label{appe2}

We first clarify the fairness of the quality-filtering comparison in the code setting, showing that the {quality-filtering baseline \textit{DCLM-fastText} in our paper does include code data in its training set}. 

Specifically, DCLM-fastText is trained with OpenHermes-2.5 (OH2.5) and reddit ELI5 as positive samples, and random documents from RefinedWeb as negative samples. OH2.5 is a high-quality QA-style dataset aggregating multiple instrunction-tuning datasets covering a diverse range of topics, which includes \textit{glaive-code-assist}, a dataset of code problems and solutions covering python, c/c++, java and many other programming languages. Statistic shows that {code data occupies 18.2\%} of all the approximately 1M samples in OH2.5.

So DCLM-fastText does include a proportion of code data in its training set. Yet, the recall of code data using DCLM-fastText is not so strong as Lev, which supports our claim that quality-filtering demands diversity for potentially higher recall of data in other domains.

To further substantiate our claim, we design a heurstic classifier to identify the code-related docs in each subset selected by different methods. Surprisingly, we find that the {insufficent recall of code data may be a more general problem among open-source data scorers}. We consider three additional quality scorers: dolma3-fasttext, qurater and nemocurater, which more or less include code data as positive examples. For subsets (select 40\% from around 30M docs) using different methods, we compare two quantities: the number of code-related docs in the selected subsets, and the code modeling ability (bit-per-byte on code benchmark) of the models trained on these subsets:

\begin{table}[ht]
\renewcommand{\arraystretch}{1.35}
\centering
\caption{Recall of code-related docs and corresponding code modeling ability under different quality scorers. BBP($\times 10^{-2}$) for Code Mean.}
\label{code_recall}
\begin{center}
\begin{small}
\begin{tabular}{l|c|c|c}
\noalign{\vskip 0.5pt}
\noalign{\hrule height 0.75pt}
\noalign{\vskip 0.5pt}

\multirow{2}{*}{Method} & Number of & Reduction of Code Docs & \multirow{2}{*}{Code Mean \textdownarrow} \\
& Code Docs & relative to Lev & \\
\hline
Lev & 1879 & -0.00\% & 62.48 \\
DCLM-fastText & 1544 & -17.83\% & 63.66 \\
NemoCurater & 1378 & -26.66\% & 63.82 \\
Random & 1003 & -46.62\% & 64.46 \\
Qurater & 837 & -55.46\% & 64.01 \\
Dolma3-fastText & 712 & -62.11\% & 65.64 \\

\noalign{\vskip 0.5pt}
\noalign{\hrule height 0.75pt}
\noalign{\vskip 0.5pt}
\end{tabular}
\end{small}
\end{center}
\end{table}

In this table, we find that the code modeling ability is positively correlated to the recall of code-related docs. The more code-related docs are retrieved, the better the code modeling ability of the pre-trained model will be.

Results show that lev retrieves the most code-related content and the corresponding pretrained model shows the best code modeling ability, while {quality filters generally retrieve less code-related content and the corresponding pre-trained models show worse code modeling ability than Lev}. Two of the quality filters retain even less code-related content than Random (i.e. even more biased away from the original distribution).

To conclude, a number of quality filters, albeit include code-domain data in their training sets, do not necessarily achieve recall as high as Lev. This {demonstrates our claim that quality filters need diversity enhancement to retrieve more code-domain data.}

\subsection{Detailed analysis of the time complexity of Lev and DiSF}
\label{detailed_time}

We provide a more detailed algorithmic comparison between Lev and DiSF. We show that the speedup comes from the faster computation of the leverage scores than that of DiSF scores.

Both algorithms follow the simple framework of greedy selection, the procedure of which is:

\begin{enumerate}
    \item Divide the corpus data into batches of size $B$.
    \item For every batch $\mathcal{B}$, it repeatedly select samples as follows:
    \begin{enumerate}
        \item Compute the DiSF/Lev scores (with respect to the currently selected samples) for every unselected sample in the batch.
        \item Select the top-$k$ samples and add them to the selected subset.
    \end{enumerate}
    \item Repeat Step 2 until the number of the selected samples reaches the budget.
\end{enumerate}

Hence, {the difference of these two algorithm only consists in Step 2(a), namely the computation of DiSF scores or Lev scores.} Next, we compare the time complexity of computing Lev and DiSF in Step 2(a).

Suppose $n$ is the number of the currently selected samples, and $X\in\mathbb{R}^{n\times d}$ is the corresponding sample matrix. Suppose the sample matrix of the currently unselected samples is $Z\in\mathbb{R}^{(B-n)\times d}$, where $B$ is the batch size and $0\leq n \leq B$.

\textbf{Step 2(a) of Lev proceeds as follows.}
\begin{enumerate}
    \item Compute $G=1/n \cdot X^\top X+\lambda I_d$, which requires $O(nd^2)$ computation.
    \item Compute the lower triangular matrix $L$, which is the Cholesky decomposition of $G=LL^\top$. This requires $O(d^3)$ computation.
    \item Solving for matrix $W\in\mathbb{R}^{(B-n)\times d}$ such that $LW^\top=Z^\top$. Because $L$ is a lower triangular matrix, solving this equations using Gaussian elimination only takes $O((B-n)\times d^2)$ steps.
    \item Compute the $\ell_2$ norm of ever row of $W$, which is exactly the leverage score of every sample in $Z$. This step requires $O((B-n)\times d)$ steps.
\end{enumerate}

Because usually $B\geq d$, substep 3 of the above is dominant. {Hence, the time complexity of Step 2(a) in the greedy selection for Lev is $O(Bd^2)$.} The key step is that through the Cholesky decomposition over $G$, the leverage scores of all the samples can be computed all together through the above substep 3 and 4.

\textbf{Step 2(a) of DiSF proceeds as follows.}
\begin{enumerate}
    \item For every unselected sample $z$, which is some row of matrix $Z\in \mathbb{R}^{(B-n)\times d}$:
    \begin{enumerate}
        \item Append $z$ to the tail of $X$. We denote the new matrix as $\hat{X}$. This is $O(1)$.
        \item Normalize every column of $\hat{X}$ (requires computation of mean and variance value of every column). This is $O(n\times d)$.
        \item Compute $C=1/n\cdot \hat{X}^\top \hat{X}$. This requires $O(n\times d^2)$.
        \item Compute the Frobenius norm of $C$: $\vert\vert C \vert\vert_F^2$. This requires only $O(d^2)$.
    \end{enumerate}
\end{enumerate}

The most computationally intensive step is Step 1(c), which requires $O(n\times d^2)$. Because we have $B-n$ unselected samples, {the overall time complexity of Step 2(a) in the greedy selection for DiSF is $O((B-n)nd^2)$.}

Because $n$ is enumerated from $1,2,...\lceil \alpha B/k \rceil$, where $\alpha$ is the selection ratio (e.g., $\alpha=40\%$), by summing up over $n$, we obtain:

\begin{itemize}
    \item For Lev, Step 2 requires $\sum_{n}O(Bd^2)=O(\alpha B^2d^2)=O(B^2d^2)$.
    \item For DiSF, Step 2 requires $\sum_{n}O((B-n)nd^2)=\sum_{n}O(Bnd^2)=O(\alpha^2 B^3d^2)=O(B^3d^2)$.
\end{itemize}

{To conclude, the time complexity of DiSF is of one order over $B$ higher than Lev. Hence, from an algorithmic perspective, Lev is more efficient than DiSF.}

\begin{remark}
    {The column-wise normalization step in DiSF is necessary.} Without normalization, it becomes $\sum_{i=1}^{d}(\lambda_i - \bar{\lambda})^2=\vert\vert C \vert\vert_F^2-(Tr(C))^2/d=\vert\vert C \vert\vert_F^2-(\sum_{i=1}^{d}\lambda_i)^2/d=\vert\vert C \vert\vert_F^2-d(\bar{\lambda})^2$. Hence, $\vert\vert C \vert\vert_F^2 = \sum_{i=1}^{d} \lambda_i^2$. {Without normalization, minimizing $\vert\vert C \vert\vert_F^2$ becomes minimizing the eigenvalue magnitude instead of optimizing eigenvalue uniformity.} In contrast, Lev does not require column-wise normalization over the Gram matrix. By conducting Cholesky decomposition over the Gram matrix only once, the leverage scores of all the samples can be computed. This is the key difference that enables higher efficiency of Lev compared to DiSF.
\end{remark}




\end{document}